\documentclass[conference]{IEEEtran}
\IEEEoverridecommandlockouts

\usepackage{amsmath,amssymb,amsfonts}
\usepackage{algorithmic}
\usepackage{graphicx}
\usepackage{textcomp}
\usepackage[svgnames]{xcolor}
\usepackage[sort,numbers]{natbib}
\usepackage[colorlinks=true, linkcolor=Navy, citecolor=Navy]{hyperref}
\usepackage{amsmath}
\usepackage{amssymb}
\usepackage{mathtools}
\usepackage{amsthm}
\usepackage{thmtools, thm-restate} 
\usepackage{multirow}
\usepackage{dirtytalk}
\usepackage{booktabs} 
\usepackage[nameinlink]{cleveref}
\usepackage{pdfpages}
\usepackage{float}

\definecolor{revisedtext}{rgb}{0.01,0.5,0.01}

\theoremstyle{plain}
\newtheorem{theorem}{Theorem}[section]

\newtheorem{corollary}[theorem]{Corollary}
\theoremstyle{definition}
\newtheorem{definition}[theorem]{Definition}

\theoremstyle{remark}

\newtheorem{aproposition}{Proposition}[section]
\newtheorem{alemma}{Lemma}[section]

\newcommand{\MSE}{\mathrm{MSE}}
\newcommand{\PSNR}{\mathrm{PSNR}}
\newcommand{\NCC}{\mathrm{NCC}} 
\newcommand{\var}{\mathrm{Var}}
\newcommand{\cov}{\mathrm{Cov}}
\newcommand{\ev}{\mathbb{E}}
\newcommand{\Bias}{\mathrm{Bias}}
\usepackage{authblk}
\makeatletter
\renewcommand{\maketitle}{
    \begin{center}
        {\color{black}\rule{\linewidth}{0.5mm}}\par\vskip 0.5em
        
        {\LARGE \textbf{\@title} \par}
        \vskip 1em
        
        {\color{black}\rule{\linewidth}{0.25mm}}\par\vskip 0.5em
        
        \@author
        
        \vskip 0.5em\par{\color{black}\rule{\linewidth}{0.5mm}}
        \vskip 1em
    \end{center}
}
\makeatother
\begin{document}

\title{From Mean to Extreme: Formal Differential Privacy Bounds on the Success of Real-World Data Reconstruction Attacks
}
\author[1,3,x]{Anneliese Riess}
\author[2]{Kristian Schwethelm}
\author[2]{Johannes Kaiser}
\author[2]{Tamara T. Mueller}
\author[1,3,4]{Julia A. Schnabel}
\author[2,5]{Daniel Rueckert}
\author[2,*]{Alexander Ziller}
\affil[1]{Institute of Machine Learning in Biomedical Imaging, Helmholtz Munich, Neuherberg, Germany}
\affil[2]{Chair for AI in Healthcare and Medicine, Technical University of Munich (TUM) and TUM University Hospital, Munich, Germany}
\affil[3]{School of Computation, Information and Technology, Technical University of Munich, Munich, Germany}
\affil[4]{School of Biomedical Engineering and Imaging Sciences, King's College London, London, United Kingdom}
\affil[5]{Department of Computing, Imperial College London, London, United Kingdom}
\affil[x]{anne.riess@tum.de}
\affil[*]{alex.ziller@tum.de}

\twocolumn[
  \begin{@twocolumnfalse}
    \maketitle
  \end{@twocolumnfalse}
]

\begin{abstract}
The gold standard for privacy in machine learning, Differential Privacy (DP), is often interpreted through its guarantees against membership inference. 
However, translating DP budgets into quantitative protection against the more damaging threat of data reconstruction remains a challenging open problem. 
Existing theoretical analyses of reconstruction risk are typically based on an \say{identification} threat model, where an adversary with a candidate set seeks a perfect match. 
When applied to the realistic threat of \say{from-scratch} attacks, these bounds can lead to an inefficient privacy-utility trade-off.

This paper bridges this critical gap by deriving the first formal privacy bounds tailored to the mechanics of demonstrated Analytic Gradient Inversion Attacks (AGIAs). 
We first formalize the optimal from-scratch attack strategy for an adversary with no prior knowledge, showing it reduces to a mean estimation problem. 
We then derive closed-form, probabilistic bounds on this adversary's success, measured by Mean Squared Error (MSE) and Peak Signal-to-Noise Ratio (PSNR). 
Our empirical evaluation confirms these bounds remain tight even when the attack is concealed within large, complex network architectures.

Our work provides a crucial second anchor for risk assessment. 
By establishing a tight, worst-case bound for the from-scratch threat model, we enable practitioners to assess a \say{risk corridor} bounded by the identification-based worst case on one side and our from-scratch worst case on the other. 
This allows for a more holistic, context-aware judgment of privacy risk, empowering practitioners to move beyond abstract budgets toward a principled reasoning framework for calibrating the privacy of their models.

\end{abstract}


\section{Introduction}
Machine learning crucially depends on access to domain-specific datasets.
However, in sensitive areas such as finance or medicine, these datasets are subject to privacy considerations.
State-of-the-art Analytic Gradient Inversion Attacks (AGIAs) can allow a malicious server or model provider to near-perfectly reconstruct sensitive training data \citep{boenisch2023curious,fowl2021robbing,feng2024privacy}.
These attacks are not merely theoretical; they are computationally inexpensive and require no special knowledge beyond the data's dimensionality, making them a practical and severe threat in real-world scenarios \citep{feng2024privacy, fowl2021robbing}.

\begin{figure}[t]
    \centering
    \includegraphics[width=0.99\linewidth]{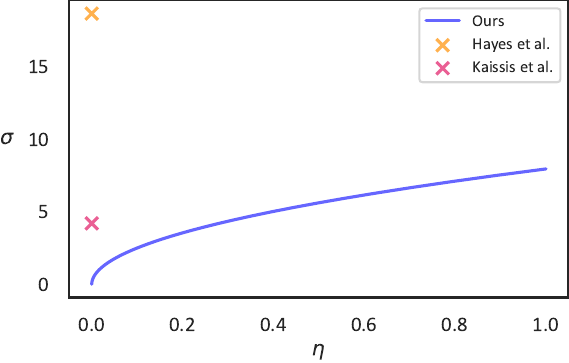}
    \caption{
    A comparison of the privacy cost required to defend against different reconstruction attack models. 
    The plot shows the amount of protective noise ($\sigma$, y-axis) needed to maintain a constant $10\%$ risk of successful reconstruction ($\gamma$), as a function of the allowable reconstruction error ($\eta$, x-axis). 
    Lower noise is better, as it leads to more accurate models.
    Our bound (blue line), for a \protect\say{from-scratch} attacker, illustrates the continuous trade-off between the allowed error and the necessary noise. 
    Unlike previous identification-based bounds \citep{hayes2023bounding, kaissis2023optimal} (markers), which prescribe a single, fixed level of privacy noise for the case of perfect reconstruction ($\eta=0$), our from-scratch based bound (blue line) introduces a continuous, two-parameter trade-off. 
    Practitioners are now empowered to explicitly set the threshold $\eta$ for what constitutes an \protect\say{acceptable} reconstruction and calibrate the noise accordingly -- enabling nuanced decisions based on actual privacy requirements and utility goals.
    The ability to set $\eta$ gives practitioners both more choice -- and more responsibility -- to precisely define and defend their desired level of privacy.
    (Parameters: Reconstruction metric is $\MSE$, $N=1$, $\Delta=1$, $\kappa=1/11$.)
    }
    \label{fig:sigmaeta}
\end{figure}

While Differential Privacy (DP) is the gold standard for privacy protection \citep{dwork2014algorithmic}, its guarantees are most commonly understood in the context of mitigating membership inference attacks \citep{nasr2021adversary}. 
Translating DP budgets into explicit, quantitative protection against the far more damaging threat of data reconstruction remains a challenging open problem \citep{balle2022reconstructing, hayes2023bounding, kaissis2023optimal}.
A key reason for this is a fundamental mismatch between existing theoretical analyses and the nature of real-world attacks.
Prior analyses are typically built upon a worst-case, \textit{identification-based} threat model: they assume a powerful adversary who already possesses a candidate set of data and whose goal is to correctly match the privatized output to the right sample \citep{hayes2023bounding, kaissis2023bounding}.
While important for understanding worst-case scenarios, this identification-based model does not reflect the objective of AGIAs. 
These attacks aim to perform \textit{from-scratch approximation}, shifting the adversarial goal from merely \say{selecting} a record from a known candidate set to \say{synthesizing} the data record itself, which may have been previously unknown to the adversary.
A fundamental analytical gap arises when one attempts to measure the risk of an approximation attack using a bound designed for identification. 
Since identification-based bounds are binary -- they only measure the probability of a perfect match -- any reconstruction with even a single altered pixel is deemed a total failure. 
Consequently, to provide any guarantee against from-scratch attacks using these bounds, a practitioner is forced to be extremely conservative. 
This leads to an overestimation of the practical privacy risk and results in a sub-optimal privacy-utility trade-off. 
For instance, a hospital might be forced to degrade a diagnostic model's accuracy with excessive noise. 
This is simply because they lack the tools to quantify the realistic threat of an informative, but imperfect, reconstruction of a patient's data.

This paper bridges this critical gap. 
We provide the first theoretical guarantees against reconstruction attacks tailored to a practical, from-scratch threat model where the adversary has no prior knowledge of the input data beyond its structure. 
We formally analyze the optimal attack strategy for such an adversary and derive tight, closed-form bounds on the probability of their reconstruction error falling below a given threshold.
As illustrated in \Cref{fig:sigmaeta}, our analysis allows practitioners to make principled decisions, calibrating noise to mitigate this specific, real-world risk.

\noindent
Our main contributions are as follows:
\begin{itemize}
    \item We formally derive the optimal AGIA strategy under a from-scratch threat model and prove that it reduces the attack to a classical mean estimation problem.
    \item We provide closed-form theoretical bounds on the attack's success, measured by probabilistic guarantees on the Mean Squared Error (MSE) and Peak Signal-to-Noise Ratio (PSNR).
    \item We demonstrate that for any given level of reconstruction risk, calibrating DP noise using our tailored bounds allows for significantly less noise -- and thus better model utility -- compared to relying on existing worst-case bounds.
\end{itemize}
The paper is structured as follows. We begin with related work and preliminaries on DP (\Cref{sec::background}), followed by a formalization of AGIAs (\Cref{sec:inversionattack}). We then present our main theoretical results (\Cref{sec::results}), contextualize them with empirical analysis (\Cref{sec::context}), and conclude with a discussion of their implications and limitations (\Cref{sec::discussion}).

\section{Background}\label{sec::background}
\subsection{Related work}\label{sec::priorwork}
Our research is situated at the intersection of two lines of work: theoretical bounds on reconstruction risk under DP, and the empirical demonstration of practical gradient-based attacks.
Our work addresses this by moving into a twofold new terrain: we provide the first theoretical bounds specifically for AGIAs, and we do so in a framework that allows for imperfect reconstruction ($\eta>0$).

The foundational analyses of reconstruction risk bounds, pioneered by \citet{hayes2023bounding}, conceptualize the adversary's task as one of \textit{identification}.
In this worst-case threat model, the adversary is assumed to possess a candidate set containing the target data point and succeeds by correctly matching the privatized output to this known sample.
\citet{kaissis2023bounding} derived a closed-form analytical bound for this identification task, improving upon initial statistical estimation techniques.
This line of work, while providing crucial worst-case guarantees, is inherently limited by its threat model.
Because the attack is framed as a binary choice -- a successful match or a failure -- the resulting analyses are restricted to evaluating the probability of perfect reconstruction.
Within the $(\eta,\gamma)$-Reconstruction Robustness (ReRo) framework of \citet{balle2022reconstructing}, where $\eta$ is an error threshold and $\gamma$ is the probability of the attack's error falling below that threshold, this corresponds to the special case where the error threshold $\eta$ is zero.
This leaves open the question of quantifying the risk of \textit{imperfect but informative} reconstructions, where $\eta > 0$.
An extension by \citet{kaissis2023optimal} considers a relaxed setting inspired by Membership Inference Attacks (MIAs), but still presupposes an oracle that provides perfect reconstructions, limiting its direct applicability.

Subsequent work has begun to explore bounds for imperfect reconstruction ($\eta > 0$). Notably, \citet{guo2022bounding} derived semantic guarantees against reconstruction by leveraging R\'enyi-DP. 
However, their analysis is not tailored to a specific attack vector, and as noted by \citet{hayes2023bounding}, these general bounds can be too loose for practical application. 
In parallel, a separate body of research has demonstrated the practical viability of Analytic Gradient Inversion Attacks (AGIAs) \citep{boenisch2023curious,fowl2021robbing,feng2024privacy}. 
These attacks perform \textit{from-scratch approximation}: they do not require a candidate set and instead aim to directly reconstruct data from gradient information. 
A significant analytical gap thus exists: prior theoretical bounds are either focused on a mismatched identification threat model, or, as in the case of \citet{guo2022bounding}, are too general to provide tight estimates for the demonstrated threat of AGIAs.
The fundamental difference between these risk landscapes -- the linear nature of identification bounds versus the surface of approximation bounds -- is visualized in \Cref{fig:comparison} in \Cref{sec::context}.

Our work introduces a methodological shift to bridge this gap. 
Instead of starting with a general privacy definition, we begin by formalizing the mechanics of practical AGIAs, deriving the statistically optimal attack strategy, and then analyzing its error distribution to yield tight, closed-form privacy bounds.
Our analysis is therefore reverse-engineered from and calibrated to this more realistic class of from-scratch attacks. 
By doing so, we derive the first tight, closed-form bounds that are specifically designed for AGIAs, allowing practitioners to quantify the risk of imperfect reconstruction ($\eta>0$) in a manner that is directly relevant to this real-world threat.

\subsection{Differential Privacy}
Differential Privacy (DP) provides a rigorous mathematical framework for analyzing data while offering strong, provable guarantees on individual privacy \citep{dwork2014algorithmic}.
The core principle is to ensure plausible deniability: the output of a randomized algorithm $\mathcal{M}$ should be so similar when run on two adjacent datasets (differing by only one individual's data) that an adversary cannot confidently determine whether any single individual's information was included in the computation.
This protection is quantified by a privacy budget, most commonly expressed as an $(\varepsilon, \delta)$-tuple.

\begin{definition}[$(\varepsilon, \delta)$-Differential Privacy)]
A randomized mechanism $\mathcal{M}: \mathcal{X} \to \mathcal{R}$ satisfies $(\varepsilon, \delta)$-differential privacy if for all adjacent datasets $D, D' \in \mathcal{X}$ (differing in one entry) and for any subset of outputs $S \subseteq \mathcal{R}$, the following inequality holds:
$ \mathbb{P}[\mathcal{M}(D) \in S] \le e^\varepsilon \mathbb{P}[\mathcal{M}(D') \in S] + \delta $.
\end{definition}
Here, $\varepsilon$ is the \textit{privacy loss}, a measure of how much the output distribution can change due to a single individual's data; smaller values imply stronger privacy. 
The parameter $\delta$ is the \textit{failure probability}, representing a small chance that the $\varepsilon$ guarantee does not hold. 
It is typically set to a small value less than the inverse of the dataset size.

\subsubsection{Additive noise mechanisms}
A primary method for achieving DP is through an additive noise mechanism \citep{dwork2014algorithmic}.
The core idea is straightforward: the true output of a function $q$ is perturbed by adding calibrated random noise before it is released.
The amount of noise required is determined by the function's sensitivity ($\Delta$), which bounds the maximum possible change in the function's output when a single individual's data is added or removed from the dataset.
A higher sensitivity implies that the function is more responsive to the presence or absence of individual data points, and thus requires more noise to obscure their contribution.

While various noise distributions can be used (e.g., Laplace, Geometric \citep{dwork2006calibrating,ghosh2009universally}), our work focuses on the Gaussian mechanism and consequently the $\ell_2$-norm as a measure of sensitivity, as this combination is the standard for privacy-preserving deep learning with DP-Stochastic Gradient Descent (DP-SGD) \citep{song2013stochastic}.

\begin{definition}[Gaussian mechanism]
Given a function $q: \mathbb{R}^M\to\mathbb{R}^N$ with sensitivity $\Delta$, the Gaussian mechanism is given as 
$\mathcal{M}(q(x))=q(x)+ \xi$, 
where $\xi$ is a sample drawn from a zero-centered $N$-dimensional Gaussian distribution with variance $\sigma^2\Delta^2$, $\sigma$ is the noise multiplier, $\Delta$ the sensitivity of $q$ and $I$ is the $N$-dimensional identity matrix.
\end{definition}
The most common sensitivity $\Delta$ used for the Gaussian mechanism is the $\ell_2$-sensitivity:
\begin{definition}[$\ell_2$-Sensitivity]
The $\ell_2$-sensitivity of a function $q: \mathbb{R}^M \to \mathbb{R}^N$ is the maximum change in the $\ell_2$-norm of its output when applied to any two adjacent datasets $D, D' \in \mathcal{X}$:
$ \Delta(q) = \max_{D, D' \in \mathcal{X}} \|q(D) - q(D')\|_2 $
\end{definition}
In practice, for functions with unbounded outputs like neural network gradients, this maximum contribution cannot be computed. 
Instead, the sensitivity is enforced by first clipping the $\ell_2$-norm of the function's output to a predefined threshold $C$, which effectively sets $\Delta = C$. 
These steps form the foundation of DP-SGD \citep{song2013stochastic}, the standard algorithm for privacy-preserving training and the mechanism we analyze throughout this work.

\subsubsection{$\mu$-GDP}
An alternative and particularly elegant formulation of DP is Gaussian Differential Privacy (GDP) \citep{dong2022gaussian}. 
GDP is parameterized by a single value $\mu = \frac{\Delta}{\sigma}$, which directly captures the privacy guarantee of the Gaussian mechanism as a function of its noise level and sensitivity.
Any $\mu$-GDP guarantee can be converted to an $(\varepsilon, \delta)$-DP guarantee for any desired $\delta > 0$. 
Prior work has shown that $\mu$ connects naturally to the analysis of membership inference attack (MIA) risk, providing a direct interpretation of the privacy budget in that context \citep{kaissis2022unified}. 
While our work does not directly use the $\mu$-GDP formalism, it follows a similar spirit: a primary goal of our analysis is to provide an analogous, intuitive interpretation of the Gaussian mechanism's parameters, specifically the noise multiplier $\sigma$, in the context of reconstruction risk.

\subsubsection{Translating Privacy Budgets into Practical Risk: $(\eta, \gamma)$-Reconstruction Robustness}
While DP provides a formal privacy budget ($\varepsilon, \delta$), this abstract guarantee can be difficult to interpret in the context of a concrete threat like data reconstruction. A practitioner might ask: \say{Given my privacy settings, what is the actual probability that an adversary can reconstruct a high-fidelity image of my data?}

To bridge this gap between abstract theory and practical risk, \citet{balle2022reconstructing} introduced the notion of \textit{$(\eta, \gamma)$-Reconstruction Robustness (ReRo)}. ReRo provides a direct, interpretable statement about the success of any reconstruction attack. It is built on two key parameters:
\begin{itemize}
    \item $\eta$: \textit{The Error Threshold.} This value defines what counts as a \say{successful} reconstruction. For a given error function (e.g., MSE), any reconstruction with an error less than or equal to $\eta$ is considered a privacy breach. A practitioner can set this threshold based on their specific use case; for example, they might decide that any reconstructed image with a PSNR above 20 dB is informatively close to the original.
    \item $\gamma$: \textit{The Success Probability.} This value is an upper bound on the probability that \emph{any} adversary can achieve a reconstruction with an error below the threshold $\eta$.
\end{itemize}
A mechanism is therefore $(\eta, \gamma)$-Reconstruction Robust if it guarantees that the probability of any adversary successfully reconstructing data with an error of $\eta$ or less is no greater than $\gamma$. This is formalized as follows:
\begin{definition}[Definition 2 in \citet{balle2022reconstructing}]\label{def::rero}
    A randomised mechanism $\mathcal{M}: \mathcal{Z}^n\rightarrow\Theta$ is $(\eta, \gamma)$-reconstruction robust with respect to a prior $\pi$ over $\mathcal{Z}$ and a reconstruction error function $l: \mathcal{Z}\times\mathcal{Z}\rightarrow \mathbb{R}_{\geq 0}$ if for any dataset $D_{-}\in\mathcal{Z}^{n-1}$ and any reconstruction attack $R: \Theta\rightarrow\mathcal{Z}$:
    $$\mathbb{P}_{Z\sim\pi,\Theta\sim\mathcal{M}(D_{-}\cup\{Z\}) }[l(Z, R(\Theta))\leq \eta]\leq\gamma.$$
\end{definition}
In essence, ReRo provides a powerful and intuitive framework for translating the abstract randomness of a DP mechanism into a concrete, quantifiable guarantee against the success of reconstruction attacks. This is the framework we use throughout our analysis to make the practical implications of our theoretical bounds clear.

\section{Analytic Gradient Inversion Attacks}\label{sec:inversionattack}

\subsection{AGIAs: a canonical example}\label{sec:inversionattackstandard}

Analytic gradient inversion attacks (AGIA) constitute the extraction of training data through the analytical inversion of model gradients.
In its simplest form, the adversary merely \textit{prepends} a linear layer to the original network architecture, without replacing any existing components.
As a result, the input data is stored in the model's gradients, and the reconstruction is achieved \textit{without} any prior knowledge of the data.
This principle forms the basis of the works by \citet{boenisch2023curious,fowl2021robbing,feng2024privacy}.
In the following, we examine this canonical case.

Consider a network whose first layer is a fully connected linear layer $f: \mathbb{R}^N \rightarrow \mathbb{R}^{M}$ with weight matrix $W\in \mathbb{R}^{M\times N}$ and bias term $b\in \mathbb{R}^{M}$.
Formally, the operation of a linear layer on an input sample $X\in\mathbb{R}^N$ can be written as
\begin{equation}
    f(X) = WX+b. \label{eq::linearlayer}
\end{equation}
Typically, such a linear layer is succeeded by other neural network operations and a loss function, which we summarize in the term $g:\mathbb{R}^N\times \mathbb{R}^M\rightarrow\mathbb{R}$. 
Hence, $g$ describes the part of the network starting from the linear layer, and $g\left(X,f(X)\right)$ denotes the network's loss function output.
Note that this formulation applies to any architecture, where the reconstruction layer $f$ gets the input $X$.

As customary, all network parameters are updated according to the loss function during each training step.
For that purpose, the gradient of $g\left(X,f(X)\right)$ with respect to all network parameters is computed by a backward pass.
We call this gradient the \textit{global, concatenated gradient} and denote it by $G_X$.
Naturally, $G_X$ is dependent on the training step and on the input sample point $X$ of that specific training step, which is evident since model updates change from one training step to another.
However, to ease notation, we do not additionally index $G_X$ with the iteration step.

For a fixed training step, besides other model updates concerning the parameters of the part of the network given by $g$, for all $j \in \{1,...,M\}$, the adversary observes 
\begin{equation}
    \nabla_{W_j}g(X,f(X)) \in \mathbb{R}^N \quad \text{and} \quad \frac{\partial g(X,f(X))}{\partial b_j} \in \mathbb{R}, \label{eq::nonprivatemodelupdatesone}
\end{equation}
namely, the gradient of $g(X,f(X))$ with respect to $j$-th row of the matrix $W$ and the derivative of $g(X,f(X))$ with respect to the $j$-th entry of the bias term $b$, respectively. 
Moreover, the adversary is aware that the gradient and the derivative in \eqref{eq::nonprivatemodelupdatesone} are constructed by a backward pass in the following way:
\begin{align}
    \nabla_{W_j} g(X,f(X))
    &= \frac{\partial g(X,f(X))}{\partial f(X)_j} \nabla_{W_j} f(X)_j \notag\\
    &= \frac{\partial g(X,f(X))}{\partial f(X)_j} X \label{eq::backwardpassnonprivateone},
\end{align}
and
\begin{align}
    \frac{\partial g(X,f(X))}{\partial b_j} 
    &= \frac{\partial g(X,f(X))}{\partial f(X)_j}\frac{\partial f(X)_j}{\partial b_j} \notag\\
    &= \frac{\partial g(X,f(X))}{\partial f(X)_j},\label{eq::backwardpassnonprivatetwo}
\end{align}
for all $j  \in \{1,...,M\}$.
Note that the gradient in Equation \eqref{eq::backwardpassnonprivateone} is a scaled version of the input $X$ and that the multiplicative factor $\frac{\partial g(X,f(X))}{\partial f(X)_j}$ on the right-hand side of Equation \eqref{eq::backwardpassnonprivateone} equals the observed update with respect to the $j$-th entry of the bias given in Equation \eqref{eq::backwardpassnonprivatetwo}. 
Thus, if there exists $j'  \in \{1,...,M\}$, such that the update $\frac{\partial g(X,f(X))}{\partial b_{j'}} \neq 0$, the adversary can reconstruct the input sample $X$ analytically by performing
\begin{equation}
    \nabla_{W_{j'}} g(X,f(X)) \oslash \frac{\partial g(X,f(X))}{\partial b_{j'}} = X, \label{eq:reconfowl}
\end{equation}
where $\oslash$ denotes the entry-wise division. 
Computing \eqref{eq:reconfowl} is possible for all $j  \in \{1,...,M\}$ such that $\frac{\partial g(X,f(X))}{\partial b_j} \neq 0$.

For completeness, we note the effect of the batch size $B$ on the attacks' success.
When the batch size $B > 1$, updates from multiple inputs are averaged together, causing them to overlap and making individual inputs difficult to reconstruct. 
While strategies by \citet{fowl2021robbing} and \citet{boenisch2023curious} attempt to mitigate this, the effect cannot be entirely prevented. 
Thus, a batch size of $B=1$ is the greatest privacy risk, as the model update is based on a single, \say{uncontaminated} input. 
Therefore, we intentionally proceed with our analysis by setting $B=1$ to bound the worst-case scenario from a privacy perspective.

\subsection{From-scratch Threat Model for AGIAs}\label{sec::threatmodel}
As we aim to construct provable guarantees against AGIA with DP, we choose our threat model to include all capabilities needed to perform state-of-the-art real-world AGIAs, e.\@g.\@, \citet{fowl2021robbing,boenisch2023curious,feng2024privacy}:

\textit{From-scratch Threat Model for AGIAs:}
Consider an adversary who is able to manipulate the deep learning setup. 
The adversary can observe intermediate gradients during training or has access to the model after training. 
Further, the adversary has no access to or knowledge of the sensitive input dataset beyond its dimensionality.

\subsection{AGIAs under DP}\label{sec:inversionattackdp}

From now on, we consider training networks of the form $g(X,f(X))$, as stated in previous \Cref{sec:inversionattackstandard}, with DP to protect the training data against AGIAs and evaluate the provided protection.
In the context of neural networks, employing DP is usually achieved by training with DP-SGD \citep{song2013stochastic}.
DP-SGD privatizes the global gradient.
To that end, the noise must be calibrated to the sensitivity of the global gradient, which is not (necessarily) bounded and can be hard to compute.
Hence, the DP-SGD algorithm is based on two main steps: (1) Clipping the $\ell_2$-norm of the global, concatenated gradient to a predefined bound $C$ in order to have an artificial bound on the sensitivity and (2) adding calibrated, zero-centered Gaussian noise to the gradient. 
The hyperparameter $C$ is called the maximum gradient norm.

For a network $g(X,f(X))$, for a fixed iteration step, the global gradient $G_X$ has the following form:
\begin{equation}
    \begin{split}
        G_X = &\left[\nabla_{W_1}g(X,f(X))^T, \frac{\partial g(X,f(X))}{\partial b_1},...\right.\\ 
         &\left...., \nabla_{W_M}g(X,f(X))^T,
        \frac{\partial g(X,f(X))}{\partial b_M}, G_{X,P}^T \right]^T,  \label{eq::globalgradient}
    \end{split}
\end{equation}
where $G_{X,P}$ denotes the concatenated gradient of $g(X,f(X))$ with respect to the rest of the parameters of the network, where all vectors are unrolled to scalars.
To induce a bound on the norm of the global gradient $G_X$, it is multiplied by the clipping term: 
\begin{equation}
    \beta_C(X) := \frac{1}{\max\left(1, \frac{\Vert G_X\Vert_2}{C}\right)}. \label{eq::definitionbeta}
\end{equation}
$\beta_C(X)$ is dependent on the iteration step by definition, and it decreases with increasing norm of the global gradient $G_X$.

Then, a noise sample $\xi$ is drawn from a multivariate Gaussian distribution $\mathcal{N}(\textbf{0}, C^2\sigma^2I)$, where $I$ is the identity matrix and $\textbf{0}$ denotes the zero vector, both matching the dimension of $G_X$.
Ultimately, $\xi$ is added to the gradient $G_X$.

Under non-DP training (see \Cref{sec:inversionattackstandard}), the gradient in Equation \eqref{eq::backwardpassnonprivateone} stores a scaled version of the target $X$.
In contrast, when employing DP, the model updates are rescaled \textit{and} also perturbed by adding Gaussian noise as follows:
\begin{alignat}{3}
    \widetilde{\nabla}_{W_j} 
    &:=s_j X + \xi_j, \quad 
    &&\xi_j \sim \mathcal{N}(\textbf{0}_N, C^2\sigma^2I_N), \label{eq::observationone}\\
    \widetilde{\nabla}_{b_j} 
    &:= s_j  + \xi_j', 
    &&\xi'_j \sim \mathcal{N}(0, C^2\sigma^2),\label{eq::observationtwo} 
\end{alignat}
for all $j  \in \{1,...,M\}$ and
\begin{equation}
    s_j := \beta_C(X)  \frac{\partial g(X,f(X))}{\partial f(X)_j}. \label{eq::definitionscalingfactor}
\end{equation}
Thus, for each $j  \in \{1,...,M\}$, Expression \eqref{eq::observationone} represents a \textit{privatized} version of $X$.  
It is easy to see that the introduced randomness impedes performing a simple division to recover the target $X$, as was possible when no DP was used (see \eqref{eq:reconfowl}).

Using the distribution of the noise, the noisy model updates in \eqref{eq::observationone} and \eqref{eq::observationtwo} can be expressed as samples from random variables in the following way
\begin{alignat}{2}
    \widetilde{\nabla}_{W_j} &\overset{d}{=} Y_j, \quad  &&Y_j\sim \mathcal{N}\left(s_j X, C^2\sigma^2I_N\right), \label{eq::distributiongradient1}\\
    \widetilde{\nabla}_{b_j} &\overset{d}{=} z_j,\;  &&z_j\sim \mathcal{N}\left(s_j, C^2\sigma^2\right),\label{eq::distributiongradient2}
\end{alignat}
for all $j \in \{1,...,M\}$, where $\overset{d}{=}$ means equal in distribution. 
Even though the adversary cannot perform the division in \eqref{eq:reconfowl} to reconstruct $X$, they can use the privatized global gradient, in particular, the observations $\widetilde{\nabla}_{W_1}, ..., \widetilde{\nabla}_{W_M}, \widetilde{\nabla}_{b_1},...,\widetilde{\nabla}_{b_M}$ and the knowledge about their distributions, as given in \eqref{eq::distributiongradient1} and \eqref{eq::distributiongradient2}, to design an estimator for the target $X$. 
Ultimately, this estimator serves as the reconstruction of $X$.
We highlight that $\widetilde{\nabla}_{W_1}, ..., \widetilde{\nabla}_{W_M}$ are $M$ observed, privatized versions of $X$ and $M$ can be chosen by the adversary since it denotes the number of rows of the matrix $W$ that specifies the linear layer $f$ (see Eq. \eqref{eq::linearlayer}). 
In \Cref{sec:optimalityattack}, we address parameter $M$'s importance for the estimator and, thus, for the attack's success. 

\subsection{The AGIA: a mean estimation problem}
\label{sec:optimalityattack}

Without making assumptions on the part of the network given by $g$, it is impossible to determine whether, for all or for some iteration steps, the part of the gradient denoted by $G_{X,P}$ (see Eq. \eqref{eq::globalgradient}) and, thereby, its privatized version contain usable information to estimate the target $X$. 
Therefore, we first focus on the privatized model updates given by $\widetilde{\nabla}_{W_1}, ..., \widetilde{\nabla}_{W_M}, \widetilde{\nabla}_{b_1},...,\widetilde{\nabla}_{b_M}$ to construct an analytically tractable estimator for $X$ and ignore the remaining part of the privatized global gradient. 
Later, we address the influence $G_{X,P}$ has on formulating an estimator for $X$.
In particular, in \Cref{propminimalvarianceestimatornew} and \Cref{prop::efficientestimatornew}, we specify and show the choice of $g$ that renders the analytic gradient inversion attack with the highest reconstruction success in terms of the $\MSE$.

Reconstructing the target $X$ using only the observations $\widetilde{\nabla}_{W_1},...,\widetilde{\nabla}_{W_M}$ is \say{not far from} solving a classical mean estimation problem.
However, the mean of the distributions of the samples $\widetilde{\nabla}_{W_1},...,\widetilde{\nabla}_{W_M}$ are not $X$ but rescaled versions of $X$, namely $s_1 X,..., s_M X$, respectively (see Eq. \ref{eq::distributiongradient1}).
Removing the dependency on the scaling factors $s_1,...,s_M$ is necessary for creating an unbiased estimator for $X$ such as the sample mean (see \Cref{prop::norealizableunbiasedestimator} in \Cref{sec::proofs}).
Since the adversary also observes noisy versions of these scaling factors, namely $\widetilde{\nabla}_{b_1},...,\widetilde{\nabla}_{b_M}$ (see \eqref{eq::distributiongradient2}), we can differentiate between two cases: one in which the adversary directly employs $\widetilde{\nabla}_{b_1},...,\widetilde{\nabla}_{b_M}$ to reconstruct $X$, and another one in which $\widetilde{\nabla}_{b_1},...,\widetilde{\nabla}_{b_M}$ are \textit{first} used to estimate the scaling factors, which are \textit{then} employed to estimate $X$.

In the former case, the adversary can divide entry-wise $\widetilde{\nabla}_{W_j}$ by $\widetilde{\nabla}_{b_j}$, for all $j\in \{1,...,M\}$ such that $\widetilde{\nabla}_{b_j}\neq 0$.
This approach mirrors the attack executed without DP constraints (see \Cref{sec:inversionattackstandard}).
The behavior of $\widetilde{\nabla}_{W_j}\oslash \widetilde{\nabla}_{b_j}$, $j\in \{1,...,M\}$, is determined by the distributions in \eqref{eq::distributiongradient1} and \eqref{eq::distributiongradient2} as follows: 
\begin{align}
    \widetilde{\nabla}_{W_j} \oslash \widetilde{\nabla}_{b_j}\overset{d}{=} V_{j}, \; V_{j,i} \sim \frac{\mathcal{N}\left(s_j x_i, C^2\sigma^2\right)}{\mathcal{N}\left(s_j,C^2\sigma^2\right)},   \label{eq::cauchydistribution}
\end{align} 
where $V_{j,i}$ denotes the $i$-th entry of $V_j$ and $V_{j,i}$ are pairwise independently distributed.
The ratio of two Gaussian distributions, as in \eqref{eq::cauchydistribution}, follows the Cauchy distribution.
Generally, this distribution has no defined statistical moments \citep{marsaglia2006ratios}, such as an expectation.
Although, under certain conditions these moments can be approximated \citep{marsaglia2006ratios,diaz2013existence}, no general statements about the behavior of the samples $\widetilde{\nabla}_{W_j}\oslash \widetilde{\nabla}_{b_j}$, $j\in \{1,...,M\}$, or the asymptotic behavior of estimators constructed using said samples, can be made if the statistical moments do not exist.
This is particularly problematic since the distribution in \eqref{eq::cauchydistribution} varies for all $j\in \{1,...,M\}$, $X\in \mathcal{D}$ and iteration steps.
Consequently, a \textit{broad} analysis is not feasible when \eqref{eq::cauchydistribution} is employed to estimate $X$.

However, alternatively, the adversary can use $\widetilde{\nabla}_{b_1},...,\widetilde{\nabla}_{b_M}$ to statistically estimate the scaling factors $s_1,...,s_M$ first and then, separately, use these estimators to rescale the observed privatized gradients \eqref{eq::observationone}, reformulate their distributions \eqref{eq::distributiongradient1} and solve the mean estimation problem for $X$.
Incidentally, the scaling factors can be approximated using strategies beyond statistical estimation.
In the simplest case, this can be achieved by imposing constraints on the data; for example, in the case of images, by assuming pixel values range from 0 to 255.

Errors in the estimation of the scaling factors $s_1, \ldots, s_M$ propagate, increasing the error in the subsequent estimation of $X$. 
Therefore, the best-case scenario for an attacker, and thus the upper bound on the attack's success, occurs when these scaling factors are estimated \textit{perfectly}. 
We proceed by analyzing this worst case from a privacy perspective, assuming the adversary has perfect knowledge of $s_1, \ldots, s_M$.

Let $X\in \mathcal{D}$ be a fixed reconstruction target input point, and also fix the training iteration step.
If $s_1,...,s_M$ are known, then the adversary can rescale the observed privatized gradients $\widetilde{\nabla}_{W_1}, ..., \widetilde{\nabla}_{W_M}$ (see \eqref{eq::observationone}), construct their sample mean $\hat{X}_M$, and use $\hat{X}_M$ as an unbiased estimator for $X$:
\begin{equation}
   \hat{X}_M := \frac{1}{M}\sum_{j=1}^M \frac{1}{s_j}\widetilde{\nabla}_{W_j}. \label{eq::firstdefestimator}
\end{equation}
Note that $\hat{X}_M$ is distributed as:
\begin{align}
    \hat{X}_M \sim \mathcal{N}\left( X, \frac{1}{M^2}\sum_{j=1}^M \frac{C^2\sigma^2}{s_j^2} I_N\right),\label{eq::distributionestimatornew}
\end{align}
because of how $\widetilde{\nabla}_{W_j}$, $j\in \{1,...,M\}$ are distributed (see \eqref{eq::distributiongradient1}). 
Performing the AGIA is, thus, in effect, equivalent to solving the mean estimation problem for $X$ using $\hat{X}_M$.

\section{Bounding the success of AGIAs}\label{sec::results}
This section forms the core of our theoretical contribution.

We move from the mean...: our analysis begins by formalizing the attack as a statistical estimation problem where the adversary's goal is to estimate the mean of a distribution, which corresponds to the true target point (\Cref{sec::optimalAGIAunderDP}).

...to the extreme: we then use this foundation to derive full probabilistic bounds on the error distribution, allowing us to quantify the risk of extreme outcomes -- that is, the probability of a reconstruction that is nearly perfect (\Cref{sec::AGIAsuccessmeasured}).

\subsection{Optimal Analytic Gradient Inversion Attack under DP} \label{sec::optimalAGIAunderDP}

Under our threat model (see \Cref{sec::priorwork}), the adversary is able to modify the model's architecture, hyperparameters, and loss function to facilitate their attack.
Therefore, they are capable of adjusting the neural network $g(X,f(X))$ to minimise the coordinate-wise variance of the sample mean $\hat{X}_M$ in \eqref{eq::distributionestimatornew}, aiming to increase the probability that $\hat{X}_{M}$ takes a value \say{close enough} to its expectation, namely to the target $X$.
The next result specifies the concrete choices of the network that render the sample mean $\hat{X}_M$ with the lowest variance:
\begin{restatable}{theorem}{propminimalvarianceestimatornew}\label{propminimalvarianceestimatornew}
    If the part of the neural network given by $g$ is replaced by the loss function $\mathcal{L}:\mathbb{R}^N\times\mathbb{R}^M\to\mathbb{R}$ with $\mathcal{L}(X,f(X))= \textbf{1}_M^Tf(X)$, where $\textbf{1}_M$ is the $M$-dimensional 1-vector, and 
    \begin{align}
        M \geq \max\left( 1, \left\lceil\frac{C}{\min_{X \in \mathcal{D}\setminus\{\textbf{0}_N\}}\Vert X\Vert_2 }\right\rceil\right),\label{eq::choiceM}
    \end{align}
    where $\lceil \cdot \rceil:\mathbb{R} \to \mathbb{N}$ denotes the function that rounds up its argument to the nearest integer,
    then $\frac{1}{M^2}\sum_{j=1}^M \frac{C^2\sigma^2}{s_j^2}$ is minimal and takes the value $\sigma^2\Vert X\Vert_2^2$. 
\end{restatable}
\begin{proof}
    See \Cref{sec::proofs}.
\end{proof}
Before addressing the significance of \Cref{propminimalvarianceestimatornew}, note that in practice, the maximum gradient norm $C$ is not typically dependent on $\min_{X \in \mathcal{D}\setminus\{\textbf{0}_N\}}\Vert X \Vert_2$. 
For instance, $C=1$ is a common choice for classification tasks or whenever practitioners want to simplify their calculations (e.g., \citet{de2022unlocking}).

Under the conditions in \Cref{propminimalvarianceestimatornew}, the sample mean $\hat{X}_{M'}$ has the minimum possible variance and is distributed as follows:
\begin{equation}
    \hat{X}_{M'} \overset{d}{=} \hat{X}, \quad \text{for} \quad \hat{X} \sim \mathcal{N}\left( X, \sigma^2 \Vert X\Vert_2^2 I_N\right), \label{eq::lastestimatorplease}
\end{equation}
for all $M'$ satisfying \eqref{eq::choiceM}.
For simplicity, we will now refer to $\hat{X}_{M'}$ as $\hat{X}$.
This estimator $\hat{X}$ \eqref{eq::lastestimatorplease} is independent of the clipping norm $C$ and its variance is \textit{exactly} calibrated to the norm of the target $X$.
From the adversary's perspective, carefully choosing $M$ according to \eqref{eq::choiceM} effectively prevents the addition of excess noise that would arise if the total gradient norm were smaller than $C$. 
However, the DP mechanism still ensures privacy because the noise multiplier $\sigma$ provides \say{just enough} protection to guarantee that $\hat{X}$ retains a non-negligible level of variability. 
Therefore, the estimate $\hat{X}$ does not converge in probability to the true target $X$.
\Cref{fig:influence_M} illustrates how $M$ impacts the reconstruction's success.

We now return to the gradient component $G_{X,P}$, which was set aside in the preceding analyses as explained at the beginning of \Cref{sec:optimalityattack}.
Due to the nature of the clipping term $\beta_C(X)$ \eqref{eq::globalgradient}, any non-usable information contained in $G_{X,P}$ contributes to the global gradient's norm.
This increases the variance of the estimator $\hat{X}$ without providing any actionable information, ultimately hindering the attack's success (see proof of \Cref{propminimalvarianceestimatornew} in \Cref{sec::proofs}).
Moreover, usable information contained in $G_{X,P}$ can be optimally incorporated into the mean estimation problem if it is a privatized version of the input $X$ as in \eqref{eq::observationone}, making $G_{X,P}$ the gradient of a linear layer.
Thus, the (insertion of the) linear layer and choosing $M$ as in \Cref{propminimalvarianceestimatornew} makes $G_{X,P}$ redundant at best.
Even if $G_{X,P}$ contains usable information, an adversary achieves the most statistically efficient reconstruction by discarding $G_{X,P}$ and instead focusing on maximizing the information from the inserted linear layer.
In other words, the following corollary is a direct consequence of \Cref{propminimalvarianceestimatornew}: 
\begin{corollary}
    Choosing $G_{X,P}\neq\textbf{0}$ can only decrease the reconstruction quality by adding to the gradient norm and triggering the addition of \say{overproportional} noise. 
\end{corollary}
\begin{proof}
    Follows directly from \Cref{propminimalvarianceestimatornew}.
\end{proof}

Additionally, \Cref{propminimalvarianceestimatornew}'s network modifications lead to $\frac{\partial g(X,f(X))}{\partial f(X)_j}=\frac{\partial \mathcal{L}(X,f(X))}{\partial f(X)_j}=1$ and
\begin{equation*}
    s_j = \beta_C(X) \qquad \forall j \in \{1,...,M\}.
\end{equation*}
Thus, the approximation of $s_1,...,s_M$ presented in \Cref{sec:inversionattackdp} is reduced to the estimation of a single scale, namely $\beta_C(X)$. 
And, most importantly, $\hat{X}$ is a sample average (constructed with independently and identically distributed observations) that estimates the expectation of a multivariate normally distributed random variable. 
Therefore, the following result also immediately follows from \Cref{propminimalvarianceestimatornew}:

\begin{restatable}{corollary}{propefficientestimatornew} \label{prop::efficientestimatornew}
    $\hat{X}$ is a minimum variance unbiased estimator (MVUE) for the target $X$.
    Thus, using the $\MSE$ as an optimality criterion, $\hat{X}$ is the best achievable estimator for the target point $X$.
\end{restatable}
\begin{proof}
    See \Cref{sec::proofs}.
\end{proof}
A MVUE is desirable because it is the most \textit{efficient} estimator: it is unbiased and achieves the Cram\'{e}r-Rao bound, i.\@e.\@, it has the lowest variability in terms of its variance compared to all other estimators for $X$ constructed with the same observations $\widetilde{\nabla}_{W_1}, ..., \widetilde{\nabla}_{W_M}$.
Moreover, such an estimator achieves the lowest expected mean squared error ($\MSE$).
In our case, a simple computation delivers: 
    \begin{equation} \label{eq::expectedmseexact}
        \ev_X[\MSE_X(X, \hat{X})] =  \sigma^2 \Vert X\Vert_2^2.
    \end{equation}
In this work, we utilize the $\MSE$ as an optimality criterion, because $\MSE_X(X, \hat{X}) = 0$ if and only if $\hat{X}$ perfectly matches the target $X$.
Therefore, by analyzing and lower-bounding the error between $X$ and $\hat{X}$, we can effectively upper-bound the adversary's success.
In particular, by \Cref{propminimalvarianceestimatornew} and \Cref{prop::efficientestimatornew}, $\hat{X}$ is the MVUE with the lowest expected $\MSE$ and variability in terms of the variance compared to all other possible MVUE obtained under other choices regarding the model's architecture, hyperparameters, and loss function.
Therefore, \Cref{propminimalvarianceestimatornew} presents the \textit{optimal attack strategy} for AGIAs, and $\hat{X}$ \eqref{eq::lastestimatorplease} is the optimal reconstruction in terms of the $\MSE$ under the assumed threat model (see \Cref{sec::threatmodel}).

The implications of the optimal AGIA can be understood as modifications to the underlying DP mechanism.
Specifically, the optimal AGIA alters both the function being privatized and the noise calibration.
Recall that a single iteration of DP-SGD is equivalent to one execution of a Gaussian mechanism $\mathcal{M}$ with appropriate privacy parameters, as detailed in \Cref{sec:inversionattackdp}. 
When an adversary executes the optimal AGIA, $\mathcal{M}$ simplifies to:
\begin{equation}
    \mathcal{M}(q(X)) = \frac{1}{\Vert X\Vert_2}X + \xi,\;\;\; \xi \sim \mathcal{N}\left(0, \sigma^2 I_N\right),
    \label{eq::modifieddpmechanism}
\end{equation}
where $q:\{\mathcal{X}\}\to \{\mathcal{X}\}$ is a function that takes only one point $X$ as input and outputs its normalized version. 
This leads us to the last implication of \Cref{propminimalvarianceestimatornew}:
\begin{corollary} \label{cor::agiaequivdpmech}
    The optimal AGIA in \Cref{propminimalvarianceestimatornew} corresponds to a reconstruction attack performed on the output of the Gaussian mechanism $\mathcal{M}$ in \eqref{eq::modifieddpmechanism}.
\end{corollary}
\begin{proof}
    Follows directly from \Cref{propminimalvarianceestimatornew} and \eqref{eq::lastestimatorplease}.
\end{proof}
This ability for an adversary to manipulate the effective privacy mechanism by engineering the model architecture is visualized in \Cref{fig:influence_M}.

\begin{figure}
    \centering
    \includegraphics[width=\linewidth]{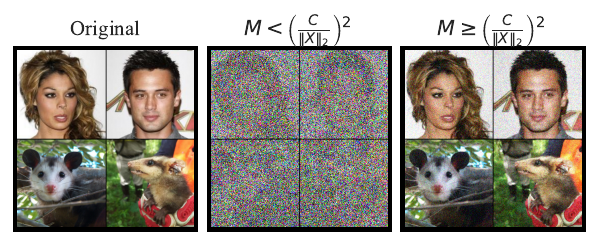}
    \caption{
    Demonstration of an adversary's ability to manipulate the privacy mechanism by engineering the model architecture. 
    The columns show reconstructions under fixed privacy parameters ($\sigma$, $C$) but different adversarial choices for the model's architecture in terms of $M$.
    \textbf{Middle}: When the adversary chooses a small $M$ such that $M < (C / \Vert X \Vert_2)$, gradient clipping is triggered on the full gradient, and the reconstruction is obscured by strong additive noise.
    \textbf{Right}: By choosing a large $M$ such that $M \geq (C / \Vert X \Vert_2)$, the adversary forces the DP mechanism to calibrate noise to a much smaller sensitivity. This results in a significantly clearer reconstruction for the same privacy budget.
    (Parameters: $\sigma=5\cdot10^{-4}$, $C=5.0\cdot10^{3}$. Middle column uses $M=1$; right column uses $M=1000$. Images from CelebA \citep{liu2015faceattributes} and INaturalist \citep{inaturalist} (under CC0-license).)
    }
    \label{fig:influence_M}
\end{figure}

\subsection{AGIA success measured by $\MSE$, $\PSNR$ and ReRo.} \label{sec::AGIAsuccessmeasured}

As a consequence of the previous section, the optimal reconstruction $\hat{X}$ obtained by an adversary performing an AGIA is inherently stochastic. 
Accordingly, the success of this reconstruction -- quantified via $\MSE$ and $\PSNR$ -- can be bounded in probability. 
By extension, the probability of a successful reconstruction \textit{for any} analytic gradient-based inversion attacks \textit{within} our defined threat model, as well as for all \textit{weaker} threat models, can also be bounded.

Fixing the target $X$, $\MSE_X(X, \hat{X})$ can be formulated as a random variable determined by the randomness of $\hat{X}$:
\begin{restatable}{theorem}{distrmsefirstcase} \label{prop::distrmsefirstcase}
        \begin{equation*}
             \MSE_X(X, \hat{X}) \overset{d}{=}  \frac{\sigma^2\Vert X\Vert_2^2}{N}\cdot Y \quad \text{with} \quad Y \sim \chi^2_N,
        \end{equation*}
    where $\chi^2_N$ denotes the central chi-squared distribution with $N$ degrees of freedom.
    In particular, for $\eta$ given, 
    \begin{align} \label{eq::almostrero}
         \mathbb{P}_{\hat{X}}(\MSE_X(X,\hat{X}) \leq \eta) = \Gamma_R\left(\frac{N}{2}, \frac{N \eta}{2\sigma^2\Vert X\Vert_2^2}\right),
    \end{align} 
    where $\Gamma_R$ is the regularized gamma function.
\end{restatable}
\begin{proof}
    See \Cref{sec::proofs}.
\end{proof}
The tail behavior of $\MSE_X(X, \hat{X})$ presented in \eqref{eq::almostrero} is particularly relevant, as the error threshold $\eta$ may be chosen such that $\MSE_X(X, \hat{X}) \leq \eta$ characterizes informative reconstructions.
\eqref{eq::almostrero} can be further generalized by eliminating the dependence on the specific target point $X$.
Moreover, by employing \Cref{cor::agiaequivdpmech}, we translate the probabilistic bound in \eqref{eq::almostrero} into reconstruction robustness:
\begin{restatable}{corollary}{reromsefirstcase} \label{prop::rerormsefirstcase}
    The DP mechanism $\mathcal{M}$ \eqref{eq::modifieddpmechanism} is $(\eta,\gamma(\eta))$-Reconstruction Robust with respect to the $\MSE$ for
    \begin{equation*}
        \gamma(\eta) = \Gamma_R\left(\frac{N}{2}, \frac{N \eta}{2\sigma^2 \min_{X \in \mathcal{D}\setminus\{\textbf{0}_N\}}\Vert X \Vert_2^2}\right).
    \end{equation*}
\end{restatable}
\begin{proof}
    See \Cref{sec::proofs}.
\end{proof}

Even though the $\MSE$ is one of the most common measures for quantifying reconstruction error, it is not robust to scaling. 
This means two reconstructions with the same $\MSE$ value can have vastly different perceptual qualities depending on the data's range.
To overcome this, the Peak Signal-to-Noise Ratio ($\PSNR$), which is based on the $\MSE$, was established as a standard metric in signal processing applications:
\begin{definition}\label{def::psnr}
    The peak signal-to-noise-ratio (PSNR) between a fixed input point $X$ and its estimator $Y_X$ is given as 
    \begin{align*}
        \PSNR_X(X, Y_X) &= 10\cdot \log_{10}\left(\frac{(\max(X)-\min(X))^2}{\MSE(X,Y_X)}\right),
    \end{align*}
    with 
    $\max(X) = \max_{i \in \{1,...,N\}} x_i, \quad \text{and} \quad \min(X)=\min_{i \in \{1,...,N\}} x_i$.
\end{definition}
The $\PSNR$ contextualizes the $\MSE$ by relating the input's range. 
This makes the $\PSNR$ robust to linear scaling, allowing for more meaningful comparisons across different scenarios.
In this regard, the $\PSNR$ has advantages over the $\MSE$ \citep{wang2009mean}.

Unlike the $\MSE$, the $\PSNR$ is a quality score, not an error metric, where higher values correspond to higher fidelity.
Since the $\PSNR$ increases as the $\MSE$ decreases, a reconstruction that minimizes the $\MSE$ (for a non-zero minimum) also maximizes the $\PSNR$.
Thus, by \Cref{propminimalvarianceestimatornew} and \Cref{prop::efficientestimatornew}, $\hat{X}$ \eqref{eq::lastestimatorplease} is also optimal with respect to the $\PSNR$.

Next, we provide probabilistic bounds for the $\PSNR$ between $X$ and $\hat{X}$ to determine its tail behavior.
To do so, we assume $\max_{X \in \mathcal{D}}\max(X)$ and $\min_{X \in \mathcal{D}}\min(X)$ are known quantities.
This information is known whenever the adversary has access to the range of the data, such as when the training data are images. 
Note that this knowledge is only required to measure the reconstruction success, but not to perform the attack.
Given that the $\PSNR$ increases as the similarity between $X$ and $\hat{X}$ grows, we derive a bound on the probability that the $\PSNR$ \textit{exceeds} a certain threshold $\eta$:

\begin{restatable}{proposition}{psnrdistributionfirstcase} \label{propo:psnrdistibutionfirstcase}
For all $X \in \mathcal{D}\setminus\{\textbf{0}_N\}$,
\begin{equation*} \label{eq::reropsnrfirstcase}\begin{split}
    \mathbb{P}_{\hat{X}}(\PSNR_X(&X,\hat{X}) \geq \eta) 
    \leq \Gamma_R\left(\frac{N}{2}, \frac{N\tilde{\eta}(\eta)}{2\sigma^2\min\limits_{X \in \mathcal{D}\setminus\{\textbf{0}_N\}}\Vert X \Vert_2^2}\right),
\end{split}
\end{equation*}
for $\tilde{\eta}(\eta) := 10^{-\frac{\eta}{10}}\left(\max\limits_{X \in \mathcal{D}}\max(X)-\min\limits_{X \in \mathcal{D}}\min(X)\right)^2$ and $\Gamma_R$ the regularized gamma function.
\end{restatable}
\begin{proof}
    See \Cref{sec::proofs}.
\end{proof}
 
To reformulate the result in \Cref{propo:psnrdistibutionfirstcase} into Reconstruction Robustness, we utilize the negative $\PSNR$ ($-\PSNR$) as a reconstruction error function:

\begin{restatable}{corollary}{psnrrerofirstcase} \label{propo:psnrrerofirstcase}
The DP mechanism $\mathcal{M}$ \eqref{eq::modifieddpmechanism} is $\left(-\eta,\tilde{\gamma}(\tilde{\eta}(\eta))\right)$-Reconstruction Robust with respect to the negative $\PSNR$ ($-\PSNR$) for any analytic reconstruction and 
\begin{equation*}
    \tilde{\gamma}(\tilde{\eta}(\eta)) =\Gamma_R\left(\frac{N}{2}, \frac{N\tilde{\eta}(\eta)}{2\sigma^2\min_{X \in \mathcal{D}\setminus\{\textbf{0}_N\}}\Vert X \Vert_2^2}\right),
\end{equation*}
for $\tilde{\eta}(\eta)$ as defined in \Cref{propo:psnrdistibutionfirstcase}.
\end{restatable}
\begin{proof}
    See \Cref{sec::proofs}.
\end{proof}

The analysis in this section marks a critical departure from prior works focused on perfect reconstruction. 
Instead of a binary success/fail outcome, our bounds allow a practitioner to compute the probability that the reconstruction error falls below any chosen threshold $\eta$. 
This is crucial from a practical privacy perspective, where an attack's success is not determined by perfect fidelity, but by whether the reconstruction is \say{sufficiently informative} to constitute a privacy breach. 
Our bounds provide the theoretical guarantee needed to quantify this nuanced risk.
An in-depth comparison of our results with the works by \citet{hayes2023bounding} and \citet{kaissis2023optimal} can be found in \Cref{sec::matching-basedbounds} and visualized in \Cref{fig:comparison}.

\subsection{Information Accumulation Across Multiple Attacks}\label{sec::multistep}
Based on a single attack, we showed in \Cref{sec:optimalityattack} that the optimal reconstruction an adversary can obtain (in terms of $\MSE$) is $\hat{X}$ \eqref{eq::lastestimatorplease}. 
We now explore how $\hat{X}$ may change under repeated adversarial attacks, where the adversary observes intermediate gradients across \textit{multiple} training iterations.

When the model is trained using DP-SGD with a fixed batch size of one, the \textit{subsampling} step of the algorithm randomly selects a single input for each iteration. 
As a result, the adversary \textit{cannot} control whether the same target record $X$ is selected across multiple attacks -- or, equivalently, training iterations.
Specifically, if the algorithm selects one input point uniformly at random from a training dataset $D$, then the probability that a particular target record is selected exactly $k$ times is given by $\binom{\hat{K}}{k} \cdot \frac{1}{|D|^k}$ when the attack is performed $\hat{K}$ times.
Thus, \textit{with probability} $\binom{\hat{K}}{k} \cdot \frac{1}{|D|^k}$, the variance of the reconstruction coordinates degrades as follows:
\begin{restatable}{proposition}{convergenceundermultipletimesteps} 
\label{convergenceundermultipletimesteps}
 Assuming the adversary can match the $k$ reconstructions $\hat{X}_1,..., \hat{X}_k$ to the same data sample $X$, they can aggregate them by averaging. 
 Let $\hat{X}_{\emph{avg}} = \frac{1}{k} \sum_{j=1}^k \hat{X}_j$ denote the averaged reconstructed vector.
Then the following holds for the $i$-th entry of $\hat{X}_{\emph{avg}}$:  
\begin{align} \label{eq::variance_average_noise}
    \ev[\hat{x_i}_{\emph{avg}}]=x_i \quad \text{and} \quad  \var\left[\hat{x_i}_{\emph{avg}}\right]&=\frac{\sigma^2\Vert X\Vert_2^2}{k}.
\end{align} 
\end{restatable}
\begin{proof}
    See \Cref{sec::proofs}.
\end{proof}
\Cref{convergenceundermultipletimesteps} was derived under the \textit{strong} assumption that the adversary is able to distinguish that their reconstructions $\hat{X}_1,..., \hat{X}_k$ correspond to the same target point $X$.
To evaluate the validity of such an assumption, it is necessary to carefully devise an adversarial strategy for identifying when multiple reconstructions correspond to the same target record.
We leave such an analysis for future work.

\subsection{Degradation of Bounds with Auxiliary Information}\label{sec::auxinfo}
A central question in assessing privacy risk is the value of auxiliary information. 
While it is clear that an adversary with more prior knowledge is more powerful, quantifying this advantage is non-trivial. 
Here, we leverage our framework to formally quantify the advantage an adversary gains from possessing a candidate set of potential data points, $\mathcal{D}$, as assumed in prior identification-based threat models \citep{hayes2023bounding,kaissis2023bounding,kaissis2023optimal}.

We measure this advantage by answering the following question: for a fixed probability of success, $\gamma$, how much does access to a candidate set reduce the achievable reconstruction error, $\eta$? 
To do this, we first need to express our bound for $\eta$ as a function of $\gamma$. 
\Cref{prop::rerormsefirstcase} can be inverted to provide exactly this relationship:

\begin{restatable}{corollary}{etafromgamma}
\label{cor:eta_from_gamma}
For a fixed probability bound $\gamma\in(\kappa,1)$ and noise variance $\sigma^2>0$, the corresponding $\MSE$ threshold $\eta$ satisfies:
\begin{equation}
    \eta(\gamma) \geq \frac{2\sigma^2}{N} \left( \min_{X \in \mathcal{D} \setminus {\{\mathbf{0}_N}\}} \Vert X\Vert_2^2 \right) \Gamma_R^{-1}\left(\frac{N}{2}, \gamma\right)
\end{equation}
where $\Gamma_R^{-1}(a,\cdot)$ denotes the inverse of the regularized gamma function with respect to its second argument.
\end{restatable}
\begin{proof}
     Follows from \Cref{prop::rerormsefirstcase} by algebraically solving for $\eta$.
\end{proof}

This corollary provides the tool to perform a direct comparison to identification-based threat models \citep{hayes2023bounding,kaissis2023bounding,kaissis2023optimal}. 
The analysis proceeds as follows:
First, consider the bounds from prior work, $\gamma_\text{prior}(\sigma,\kappa)$, which give the probability of a \textit{perfect} reconstruction ($\eta=0$) for an adversary with a candidate set.
Next, we can insert this same success probability, $\gamma_\text{prior}$, into our inverted formula from \Cref{cor:eta_from_gamma}.
The result, $\eta(\gamma_\text{prior})$, tells us the reconstruction error an adversary \textit{without} a candidate set has to tolerate to achieve the exact same probability of success.

Therefore, this value $\eta(\gamma_\text{prior})$ precisely quantifies the benefit of having a candidate set: it is the \say{error discount} that auxiliary information provides. 
It represents the reduction in reconstruction error -- from some non-zero value $\eta > 0$ all the way down to $\eta=0$ -- that an adversary gains by moving from a from-scratch attack to an identification-based attack (see \Cref{fig:comparison}). 
We call the interval $\left[0,\eta(\gamma_\text{prior})\right]$ \say{risk corridor}.

\section{Calibrated protection enables favorable privacy-utility trade-offs}\label{sec::context}
\begin{figure}[t]
    \centering
    \includegraphics[width=0.99\linewidth]{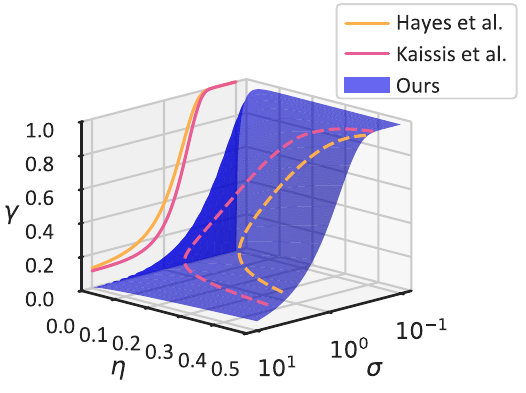}
    \caption{
    A comparison of the risk surfaces for from-scratch reconstruction (our bound) versus identification-based reconstruction (prior work). The axes represent the noise multiplier ($\sigma$), the reconstruction error threshold ($\eta$), and the attack success probability ($\gamma$).
    The bounds from \citet{hayes2023bounding} and \citet{kaissis2023optimal} are lines restricted to the $\eta=0$ plane, as they only model the risk of perfect reconstruction (identification).
    In contrast, our bound (blue surface) characterizes the risk for any error threshold $\eta > 0$, reflecting the nature of from-scratch approximation.
    The dashed lines show the projection of the prior bounds onto our risk surface. This projection visualizes the \protect\say{error discount} an adversary gains from possessing a candidate set: for a fixed success probability $\gamma$, the corresponding point on our surface reveals the reconstruction error $\eta$ a from-scratch attacker would have to tolerate.
    (Parameters: $\MSE$ metric, $N=1$, $\Delta=1$, $\kappa=1/10$.)
    }
    \label{fig:comparison}
\end{figure}

This section contextualizes our theoretical results by comparing them against the matching-based bounds proposed by \citet{hayes2023bounding}, \citet{kaissis2023bounding}, and \citet{kaissis2023optimal}. 
We analyse the fundamental differences in threat models and their interpretations by visualizing the effect of a candidate set on reconstruction robustness (analogous to \Cref{sec::auxinfo}). 
Furthermore, we empirically evaluate the tightness of our bounds under practical deviations from the optimal attack strategy.
The program code, which was used to create all the Figures, can be found at \url{https://github.com/TUM-AIMED/FromMeanToExtreme}.

\subsection{Comparison to matching-based bounds.} \label{sec::matching-basedbounds}
A visual comparison of our from-scratch reconstruction bound with the matching-based bounds of \citet{hayes2023bounding} and \citet{kaissis2023optimal} is presented in \Cref{fig:comparison}, which plots the relationship between the noise multiplier ($\sigma$), reconstruction error ($\eta$), and attack success probability ($\gamma$).
This comparison reveals several fundamental differences rooted in the underlying threat models.

First, the two classes of bounds quantify different adversarial goals: \textit{identification} versus \textit{approximation}. 
Matching-based bounds are defined only for the case of perfect reconstruction ($\eta=0$), evaluating the probability that an adversary can correctly identify the true data point from a known candidate set. 
In contrast, our bound characterizes the probability of achieving an imperfect reconstruction with an error less than or equal to $\eta$. 
A direct consequence is that for our bound, the probability of a perfect reconstruction ($\eta=0$) is always zero for any injected noise ($\sigma>0$). 
This is because the continuous Gaussian noise added to the gradient would have to be exactly zero for a perfect reconstruction.

Second, projecting the matching-based bounds onto the surface of our bound, as described in \Cref{sec::auxinfo}, allows us to visualize the effect of auxiliary knowledge (i.e., a candidate set) on the achievable reconstruction error. 
For moderate values of attack success probability ($\gamma$), the required reconstruction error ($\eta$) is similar between the models. 
However, at the extremes, the difference diverges. 
As $\gamma \to 1$, our bound converges more slowly, indicating that achieving near-certain reconstruction requires tolerating a higher error $\eta$ in a from-scratch attack compared to a matching attack. 
The most telling divergence occurs as the random guess baseline is approached. 
Matching-based bounds assume the target is within a known candidate set, so an adversary can achieve a baseline success rate of $\gamma=\kappa$ (where $\kappa$ is the inverse of the candidate set size) simply by guessing randomly, even with no information from the model ($\sigma \to \infty$). 
In our from-scratch model, an adversary without a candidate set has no such advantage; a random guess in a high-dimensional space results in a negligible success probability. 
Consequently, as $\sigma \to \infty$, the success probability for our bound approaches zero ($\gamma \to 0$). 
This fundamental difference in baseline assumptions causes the required reconstruction error $\eta$ to diverge substantially as $\gamma$ approaches the random-guessing floor $\kappa$ of the matching-based models.

\subsection{Testing our bounds in real-world situations}\label{sec::tightness}
\begin{figure}[t]
    \centering
    \includegraphics[width=\linewidth]{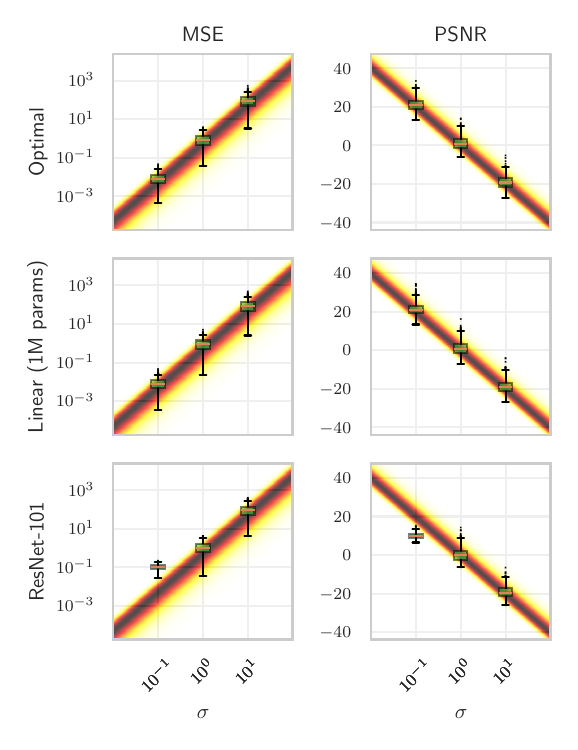}
    \caption{
    Empirical reconstruction quality vs. our theoretical probabilistic bounds. 
    Each subplot shows the reconstruction error (y-axis) as a function of the noise multiplier $\sigma$ (x-axis).
    The background color map illustrates the Probability Density Function (PDF) predicted by our theory (darker is higher probability), while the overlaid box plots show the distribution of empirical results from 500 samples.
    The columns compare two metrics (MSE and PSNR). 
    The rows represent increasingly complex and realistic attack scenarios: the top row shows the ideal case matching our theoretical optimal attack; the middle row adds a non-contributory 1M-parameter linear layer; and the bottom row simulates a practical attack embedded within a full ResNet-101 architecture ($\approx$45M parameters).
    The close alignment between the empirical distributions and the theoretical PDFs across all scenarios validates the tightness of our bounds, even in realistic, high-parameter settings.
    (Fixed parameters: $C=1$, $M=1$.)
    }
    \label{fig:tightness}
\end{figure}
Our theoretical bounds are derived to reflect practical attack scenarios.
To this end, we defined an optimal attack strategy under the AGIA threat model, which involves replacing the network with a single linear layer.
However, existing attacks, such as those by \citet{boenisch2023curious,feng2024privacy,fowl2021robbing}, embed the reconstruction components within larger networks.
As established in \Cref{sec::results}, additional layers that contribute to the total gradient norm without aiding in reconstruction increase the noise required under DP, which can degrade the quality of the reconstruction.

To quantify this effect and evaluate the tightness of our bounds, we conduct an empirical analysis. 
We evaluate reconstruction quality across three configurations: 1) the optimal attack strategy; 2) the optimal attack augmented with a non-contributory linear layer of one million parameters; and 3) the attack architecture embedded within a ResNet-101, adding approximately 45 million non-contributory parameters.
\Cref{fig:tightness} plots the results, and implementation details are provided in \Cref{sec::detailtightness}.

The empirical results show a close correspondence with the distributions predicted by our theoretical bounds.
A deviation is observed for small values of $\sigma$ in the ResNet-101 setting, where the empirical reconstruction quality is lower than predicted by the bounds.
In the other settings, the empirical results align with the theoretical predictions. 
This indicates that the derived bounds provide a useful estimate of reconstruction risk, even when the attack is situated within larger network architectures.

\section{Discussion \& Conclusion}\label{sec::discussion}
The ultimate value of theoretical bounds against reconstruction attacks lies in their practical utility. 
The central question is whether such bounds can guide practitioners in the difficult task of reconciling stringent privacy requirements with the need for high-performing machine learning models.
For these guarantees to be meaningful, they must accurately describe reconstruction robustness in realistic, not just theoretical, settings.
This work contributes to this goal by providing the first formal bounds calibrated to the mechanics of real-world Analytic Gradient Inversion Attacks (AGIAs) -- those demonstrated by \citet{boenisch2023curious,fowl2021robbing,feng2024privacy} -- rather than relying on hypothetical, analytically convenient threat models.

The primary contribution of our analysis is the empowerment of the practitioner.
While prior worst-case bounds provide important theoretical guarantees, they become overly pessimistic when misapplied to more realistic threats like AGIAs, forcing the addition of excessive noise that degrades model accuracy (see \Cref{fig:sigmaeta}).
Our results offer a more fine-grained tool by introducing the reconstruction error threshold, $\eta$, as an explicit parameter in the risk analysis.
By this, a practitioner can now quantitatively estimate the probability of an adversary's reconstruction error falling below a specific success threshold $\eta$ for a given dataset and a chosen set of DP parameters ($\sigma, C$).
This transforms the abstract concept of privacy into a concrete, actionable risk assessment.
Equipped with this capability, practitioners can make a more informed decision, selecting a noise multiplier $\sigma$ that provides a desired level of protection against reconstruction while minimizing the accompanying impact on model utility.

Consider a medical imaging application where a reconstruction $\MSE$ below $\eta\leq0.5$ was identified as a critical threshold where errors below could potentially reveal patient-identifiable features. 
Using our bounds, a hospital's data science team could determine the minimum $\sigma$ required to ensure that the probability of such a reconstruction remains below, for instance, $\gamma\leq10^{-3}$. 
This allows them to either proceed with confidence or conclude that the utility cost for the required level of privacy is too high, prompting a re-evaluation of the project's feasibility.

A key reason our bounds offer a different perspective is our threat model's assumption of a \say{from-scratch} reconstruction, where the adversary possesses no prior candidate set. 
This stands in contrast to prior bounds based on matching a privatized output to a known set of candidates \citep{hayes2023bounding, kaissis2023bounding,kaissis2023optimal}. 
While the matching model represents a valid worst-case scenario, our from-scratch model more accurately reflects the capabilities of many demonstrated AGIAs, specifically those that operate without strong, learned priors about the data distribution.
It is important to acknowledge, however, that reality often lies between these two extremes: an adversary may not possess a full candidate set, but may have some form of auxiliary information.
Our framework provides a crucial tool for reasoning about this uncertainty. 
By establishing a tight, worst-case bound for the from-scratch threat model (i.e., an adversary with no prior candidate set), we provide practitioners with a formal guarantee for a realistic and commonly demonstrated class of attacks. 
This allows the overall risk to be bounded from two sides: the identification-based bounds provide the worst-case risk assuming a powerful adversary with a full candidate set, while our analysis provides the worst-case risk assuming a no-prior-knowledge adversary. 
A data owner can now assess this \say{risk corridor} -- the space between the two formal guarantees -- to make a more nuanced judgment about the practical privacy risk, rather than relying on a single, potentially misaligned worst-case number.

Analyzing the boundaries of this risk corridor reveals an interesting and seemingly counter-intuitive phenomenon, as visualized in \Cref{fig:sigmaeta}. 
While our bound offers a clear advantage for small, non-zero error thresholds ($\eta$), there exists a crossover point beyond which our bound requires \textit{more} noise to guarantee a given risk level $\gamma$ for $\gamma>\kappa$ than the identification-based bounds.
This occurs because the two bounds are fundamentally protecting against different adversarial tasks. 
The identification-based adversary's task is constrained: they must find a perfect match within a finite candidate set.
In contrast, our from-scratch adversary faces an unconstrained task: they can succeed if their reconstruction is \say{close enough} to the target in an infinite space of possibilities. 
For a very permissive error threshold $\eta$, the volume of \say{successful} reconstructions becomes vast, making it statistically harder to guarantee that the adversary's result will fall outside this large volume. 
Therefore, preventing any reconstruction with an error less than a large $\eta$ can be a harder problem to solve -- and thus require more noise -- than preventing a perfect match from a small, pre-defined set. 
This highlights that the choice of an appropriate threat model is not just about pessimism, but about accurately defining the specific adversarial capability one aims to defend against.

Our theoretical analysis is centered on Mean Squared Error ($\MSE$) and Peak Signal-to-Noise Ratio ($\PSNR$), which are standard metrics for reconstruction quality and are directly amenable to probabilistic bounding within the $(\eta, \gamma)$-ReRo framework. 
However, for certain diagnostic purposes, other metrics can provide complementary insights. 
Notably, the Normalized Cross-Correlation ($\NCC$) is particularly useful for detecting cases where the adversary correctly reconstructs the structure of the data but fails to estimate the correct scaling factor, a scenario where $\MSE$ might be misleadingly high. 
However, the $\NCC$ has its own limitations: its value is sensitive to data dimensionality, making cross-scenario comparisons difficult, and as a single correlation statistic, it does not naturally yield the kind of probabilistic error bounds central to our framework. 
Acknowledging its diagnostic utility, we include a formal analysis of the $\NCC$ and its bounds in \Cref{sec::ncc}.

Our analysis is intentionally focused on the batch size one ($B=1$) scenario, as this constitutes a necessary and foundational step in understanding reconstruction risk. 
This choice is motivated by its role as the privacy worst-case for this attack vector. 
The averaging of gradients over a larger batch naturally obscures the contribution of any single data point, making direct reconstruction significantly more difficult. 
By providing a formal bound for the $B=1$ case, we establish a robust guarantee against the most potent and direct form of this threat.
This focus is further justified by the mechanics of demonstrated AGIAs. 
For example, the attack by \citet{feng2024privacy}, which uses \say{privacy backdoors}, is explicitly designed to create \say{data traps} that activate on a \textit{single} input within a batch and then become inactive. 
Even in their more complex model, the goal -- and the fundamental threat -- is the isolation and reconstruction of a \textit{single} sample. 
Similarly, other sophisticated techniques for distilling \textit{single} samples from larger batches have shown high failure rates in practice \citep{boenisch2023curious,fowl2021robbing}.
Therefore, analyzing the $B=1$ case is not a simplification for convenience, but rather a direct analysis of the core adversarial goal. 
Analogously, our assumption that the adversary can accurately estimate the scaling factors also represents a worst-case analysis. Any error in estimating these scales would propagate and increase the error of the final reconstruction, making the attack less effective. Therefore, to derive a robust upper bound on the attacker's success, we must analyze the scenario where their estimation is perfect.
While extending these formalisms to model the statistical effects of larger batches or imperfect scale estimation are important areas for future research, any such analysis will necessarily build upon the foundational understanding of the single-sample, perfect-information worst-case that our work provides.

A natural desire arising from this line of work is for a prescriptive \say{cookbook} that provides concrete privacy parameters for specific applications. 
However, a key implication of our analysis is that such a one-size-fits-all approach is fundamentally misguided. 
The very purpose of our framework is to move the field away from a single, universal number and toward a more nuanced, context-dependent form of reasoning.
Our bounds do not provide a single answer, but rather a principled methodology for a practitioner to find their own answer. 
The process involves: (1) defining what level of reconstruction error ($\eta$) is considered harmful for a specific use case -- a decision that requires domain expertise; (2) specifying an acceptable risk tolerance ($\gamma$); and (3) using our derived bounds to calculate the necessary noise multiplier ($\sigma$) to enforce that specific guarantee. 
Therefore, our contribution is not a recipe book, but rather the provision of mathematical tools and a conceptual framework required for practitioners to engage in this essential, context-aware risk assessment themselves.

Despite these advances, our work is not without limitations, and it highlights several important avenues for future research.
While our framework provides probabilistic guarantees for common error metrics like $\MSE$ and $\PSNR$, it does not solve the fundamental problem of defining what constitutes a \say{successful} reconstruction.
The choice of an acceptable error threshold is inherently context-dependent; a reconstruction that is unintelligible for one application might be sufficiently informative to constitute a privacy breach in another.
Therefore, the application of our bounds requires practitioners to define this threshold based on the specific data and circumstances of their use case.

A crucial question concerns the scope of our threat model. 
While our work provides the first tight bounds for a \say{from-scratch} adversary, it is important to clarify that our analysis is calibrated specifically to the mechanics of Analytic Gradient Inversion Attacks (AGIAs) demonstrated in prior work \citep{boenisch2023curious,fowl2021robbing,feng2024privacy}. 
Our threat model intentionally assumes an adversary with minimal auxiliary information beyond the data's basic structure, which reflects the capabilities of these published attacks.
However, this is not the strongest conceivable from-scratch adversary. 
More powerful adversaries might leverage sophisticated, learned priors about the data distribution. 
For instance, recent work by \citet{schwethelm2024visual} has shown that incorporating generative priors from diffusion models can significantly enhance the success of reconstruction attacks. 
Our current bounds do not account for such powerful auxiliary knowledge and would not hold against this stronger threat. 
Therefore, while our framework provides a robust defense against a demonstrated class of real-world attacks, extending this analysis to quantify the risk posed by adversaries equipped with powerful learned priors represents a critical and challenging frontier for future research.

Our analysis of information accumulation over multiple attack iterations rests on a strong assumption: that the adversary can perfectly match multiple, noisy reconstructions to the same underlying data record.
While this provides a useful upper bound on the risk of repeated attacks, a real-world adversary may struggle with this matching task.
A crucial direction for future work is to analyze how these bounds degrade under a more realistic assumption of imperfect matching. 
Such an analysis would provide an even more complete picture of long-term privacy risks in dynamic, multi-iteration settings like DP-SGD.

Ultimately, this line of research points toward a future where privacy guarantees are not one-size-fits-all, but are instead understood in the context of specific, well-defined threat models. 
Our work is a critical step in this direction. 
Instead of relying on a single, abstract $(\varepsilon,\delta)$-budget, a practitioner can now use our bounds in concert with prior work to understand how the risk of reconstruction itself changes based on adversarial assumptions. 
For a given privacy setting $\sigma$, $\Delta$, they can formally state the worst-case risk against a \say{from-scratch} AGIA adversary (using our bound) and the worst-case risk against an \say{identification} adversary (using prior bounds). 
This ability to map a single privacy setting to a range of concrete outcomes enables a far more holistic and actionable risk assessment, moving the field from abstract budgets to context-aware security decisions.

In conclusion, by shifting the focus from worst-case theoretical adversaries to the mechanics of practical attacks, this work provides a more nuanced and actionable framework for managing reconstruction risk in modern machine learning.

\section*{Acknowledgment}
We thank George Kaissis for the productive discussion of ideas in this manuscript.
\section*{LLM usage considerations}
LLMs were used for editorial purposes in this manuscript, and all outputs were inspected by the authors to ensure accuracy and originality.


\bibliographystyle{plainnat}
\bibliography{main}

\appendix
\subsection{Reconstruction Success Measured by the NCC} \label{sec::ncc}
The normalized-crossed correlation ($\NCC$) is a metric that measures the linear correlation between a target point $X=(x_1,...,x_N)^T$ and its estimated counterpart $Y_X= (y_{1,X},...,y_{N,X})^T$ instead of measuring their difference. 
Hence, as opposed to metrics such as the $\MSE$ or $\PSNR$, it is robust to linearly transformed inputs. 
The $\NCC$ is defined as follows:
\begin{definition}\label{def::ncc}
    The normalized cross-correlation ($NCC$) \citep{rodgers1988thirteen} between a fixed input point $X$ and its estimator $Y_X$ is defined as
    \begin{align*}
        \NCC(X, Y_X) &= \frac{\cov(X,Y_X)}{\sigma_X\sigma_{Y_X}},
    \end{align*} 
     where 
    \begin{equation*}
       \cov(X,Y_X) = \frac{1}{N}\sum_{i=1}^N (x_i-\ev_x[x])(y_{i,X}-\ev_{y_X}[y_X]),
    \end{equation*} 
    for $\ev_x[x], \ev_{y_X}[y_X]$ denoting the numerically obtained expected values within a sample.
    Moreover, $\sigma_X =\sqrt{\var(X)}$ and $\sigma_{Y_X}=\sqrt{\var(Y_X)}$, for $\var(X)$ and $\var(Y_X)$ being the sample variances of $X$ and $Y_X$, respectively. 
\end{definition}
We note that the $\NCC$ is equivalent to the Pearson's Correlation Coefficient, which is common in statistical testing.
High values of the $\NCC$ denote a high linear correlation between $X$ and $Y_X$, which, depending on the context, can imply a high similarity between $X$ and $Y_X$. 


We apply the following strategy to compute the (theoretical) correlation given by the $\NCC$ between a target $X$ and its reconstruction $\hat{X}$ \eqref{eq::lastestimatorplease}:
First, let the target $X$ be fixed.
We assume there exists a continuous, one-dimensional random variable $x$ distributed in such a way that $\{x_1, ... , x_N\}$ are probable samples from this distribution. 
In particular, we let $[\min_{i \in \{1,...,N\}}x_i, \max_{i \in \{1,...,N\}}x_i]$ be the support of $x$ and set $\var(x) \leq \Vert X \Vert_2^2/N$.
The latter assumption is motivated by the following fact: $X$ is an $N$-dimensional vector, hence, drawing a random element from its entries $\{x_1, ... , x_N\}$ can be represented by a discrete, uniformly distributed random variable with support $\{x_1, ... , x_N\}$ and variance bounded by $\Vert X \Vert_2^2/N$.
However, for our analysis, $x$ needs to be a continuous random variable.
Hence, we define $x$ to be continuous, but maintain the range of the support and the variance of its discrete counterpart.
Analogously, we construct the continuous, one-dimensional random variable $\hat{x}$ such that $\hat{x} :=  x +\zeta$, for $\zeta \sim \mathcal{N}(0, \sigma^2\Vert X \Vert_2^2)$ independent of $x$.
By definition, $\hat{x}_{1},...,\hat{x}_N$ are probable samples of the random variable $\hat{x}$.
It is easy to see, that measuring the correlation between $x$ and $\hat{x}$ is equivalent to measuring the correlation between $X$ and $\hat{X}$.

\begin{restatable}{proposition}{nccfirstcase} \label{prop:nccfirstcase}
Let $x$ and $\hat{x}$ be the two random variables as defined above.
Then,
   \begin{equation} 
         \NCC ( x, \hat{x} ) 
         =\sqrt{\frac{1}{1 + \sigma^2\Vert X\Vert_2^2/\var(x)}} 
         \leq \sqrt{\frac{1}{1 + \sigma^2 N}}. 
         \label{eq::nccfirstcasesinglexi} 
    \end{equation}
\end{restatable} 
\begin{proof}
Using the definition of the $\NCC$ (see Definition \ref{def::ncc}) it follows:
    \begin{align}
        \NCC\left( x, \hat{x}\right) 
        &= \frac{\cov\left( x, \hat{x}\right)}{\sqrt{\var(x)}\sqrt{\var(\hat{x})}} \notag \\
        &=\frac{\cov\left( x, x\right) + \cov\left( x, \zeta\right)}{\sqrt{\var(x)}\sqrt{\var(x) + \var(\zeta)}} \notag \\
        &= \frac{\var(x)}{\sqrt{\var(x)}\sqrt{\var(x) + \var(\xi)}} \notag \\
        &= \sqrt{\frac{\var(x)}{\var(x) + \var(\zeta)}}\notag \\
        &= \sqrt{\frac{1}{1 + \var(\zeta)/\var(x)}}. \label{eq::boundnccncc}
    \end{align}

We recall that $\var(\zeta) = \sigma^2\Vert X\Vert_2^2$ by definition.
Thus, using the assumption that $\var(x)\leq \Vert X\Vert_2^2/N$, it follows from \eqref{eq::boundnccncc}:
\begin{equation*}
    \NCC\left( x, \hat{x}\right) = \sqrt{\frac{1}{1 + \sigma^2\Vert X\Vert_2^2/\var(x)}} \leq \sqrt{\frac{1}{1 + \sigma^2 N}}.
\end{equation*}
\end{proof}
According to \Cref{prop:nccfirstcase}, the $\NCC(x,\hat{x})$ depends on a key ratio: the reconstruction's internal variance ($\sigma^2 \Vert X \Vert_2^2$) divided by the original data's internal variance ($\var(x)$).
If this ratio is too low, it means the noise from the DP mechanism is not strong enough to break the linear relationship between $x$ and $\hat{x}$, or, conversely, between $X$ and its reconstruction $\hat{X}$.
Because of how the reconstruction $\hat{X}$ is defined (see \eqref{eq::lastestimatorplease}), a high correlation between $x$ and $\hat{x}$ implies $\hat{X}$ is a high-quality reconstruction of the target $X$.
Specifically, the $\NCC$ equals one (indicating a perfect linear relationship) if and only if the noise multiplier $\sigma$ is zero. 
This happens when no privacy-preserving noise is added. 
Incidentally, any non-linear transformation of the privatized gradient would distort this linear relationship. The reconstruction $\hat{X}$ is also optimal with respect to the $\NCC$.

For any non-zero noise ($\sigma > 0$), the $\NCC$ decreases as the noise level increases. 
However, \Cref{prop:nccfirstcase} shows that the $\NCC$ \textit{far more} dependent on the data's dimension, $N$.
Because the $\NCC$ is on the order of $\mathcal{O}(\sqrt{1/N})$, the correlation naturally approaches zero in high-dimensional settings, regardless of the specific amount of noise.
This has an essential implication: for high-dimensional data, a very low $\NCC$ value does not necessarily mean a poor reconstruction.
Instead, the $\NCC$ must be compared against a case-specific threshold, $\eta$, to determine if a reconstruction is informative.
For instance, if only reconstructions with $\NCC$ values exceeding $\eta$ are considered informative in a specific context, and $\eta$ is greater than the right-hand side of \eqref{eq::nccfirstcasesinglexi}, then no reconstruction constructed by the adversary can be deemed useful.
However, a key limitation of the $\NCC$ is that it cannot be used to compare the reconstruction success across scenarios with differing input dimensionalities $N$.

\subsection{Proofs}\label{sec::proofs}
In the following, we give the proofs for our theoretical results. 

\begin{aproposition}
\label{prop::norealizableunbiasedestimator}
     If the scaling factors $s_1,...,s_M$ are unknown, then there is no \say{realisable} unbiased estimator of the target $X$ that can be constructed solely using the observed privatized gradients $\widetilde{\nabla}_{W_1},...,\widetilde{\nabla}_{W_M}$. 
\end{aproposition}
\begin{proof}
    First, we show \textit{there is no deterministic transformation $T_X:\mathbb{R}^N\to \mathbb{R}^N$, such that $T_X\left(\widetilde{\nabla}_{W_j}\right)$, $j\in\{1,...,M\}$, is a \say{realisable} unbiased estimator of $X$.}
    
    Let $X$ be a fixed reconstruction target point and without loss of generality (w.l.o.g.) assume $X\neq \textbf{0}_N$, for $\textbf{0}_N$ the $N$-dimensional zero vector.
    Moreover, for a fixed $j\in\{1,...,M\}$, let $T_X\left(\widetilde{\nabla}_{W_j}\right)$ be a realisable, unbiased estimator of $X$.
    Since 
    \begin{equation*}
        \widetilde{\nabla}_{W_j} = s_jX + \xi_j, \qquad \xi_j \sim \mathcal{N}(\textbf{0}_N, C^2\sigma^2I_N),
    \end{equation*}
    $T_X\left(\widetilde{\nabla}_{W_j}\right)$ must be an affine function, because that is the only transformation that can invert an affine function.
    Hence, $T_X\left(\widetilde{\nabla}_{W_j}\right)$ has the following form:
    \begin{equation}
        T_X\left(\widetilde{\nabla}_{W_j}\right) = A \widetilde{\nabla}_{W_j} + b,
    \end{equation}
    for $A\in \mathbb{R}^{N\times N}$ invertible and constant, and $b\in \mathbb{R}^{N}$ constant. By definition of realisable estimators, $A$ and $b$ cannot be functions of $X$. 
    Next, we compute the expectation of $T_X\left(\widetilde{\nabla}_{W_j}\right)$ for $X$ fixed:
    \begin{align*}
        \ev_X\left[T_X\left(\widetilde{\nabla}_{W_j}\right)\right] 
        &= \ev_X\left[T_X\left(s_jX + \xi_j\right)\right]\\
        &= \ev_X\left[As_jX + A\xi_j+ b\right]\\
        &= As_jX +b.
    \end{align*}
    $T_X\left(\widetilde{\nabla}_{W_j}\right)$ is an unbiased estimator of $X$ if and only if $\ev_X\left[T_X\left(\widetilde{\nabla}_{W_j}\right)\right]=X$, which is equivalent to the following:
    \begin{align}
         As_jX +b = X 
        \iff  \left(s_jA-I_N\right) X + b = \textbf{0}_N.\label{eq::prop6independencefromx}
    \end{align}
    Since $b$ cannot depend on $X$, we cannot set $\left(s_jA-I_N\right) X = -b$ and solve the equation. 
    Therefore, the equations in \eqref{eq::prop6independencefromx} are satisfied if and only if
    \begin{align}
    \left(s_jA-I_N\right) X  &= \textbf{0}_N \quad \text{and} \quad b = \textbf{0}_N.
    \end{align}
    Since $X\neq \textbf{0}$, then $As_jX +b = X $ holds if and only if
    \begin{equation*}
        A = \frac{1}{s_j} I_N \quad \text{and} \quad b=\textbf{0}_N,
    \end{equation*}
    implying that $A$ depends on $s_j$, a quantity that can only be computed knowing $X$ (see \eqref{eq::definitionscalingfactor}) and contradicting the assumption that $T_X\left(\widetilde{\nabla}_{W_j}\right)$ is a realisable estimator. 
    Therefore, for any $j\in\{1,...,M\}$, there is no realisable unbiased estimator of $X$ that can be constructed using $\widetilde{\nabla}_{W_j}$.

    If there was an unbiased estimator of $X$ constructed with $\widetilde{\nabla}_{W_1},...,\widetilde{\nabla}_{W_M}$, then there would exists at least one $j\in\{1,...,M\}$ such that an unbiased estimator of $X$ can be constructed with $\widetilde{\nabla}_{W_j}$.
    However, since there is no $j\in\{1,...,M\}$ such that an unbiased estimator of $X$ can be constructed with $\widetilde{\nabla}_{W_j}$, it follows by contraposition, that there is no unbiased estimator of $X$ than can be constructed with $\widetilde{\nabla}_{W_1},...,\widetilde{\nabla}_{W_M}$.
\end{proof}

The following proposition serves as an auxiliary result to obtain \Cref{propminimalvarianceestimatornew}:
\begin{aproposition}
\label{prop::weaklawoflargenumbersnothold}
The coordinte-wise variance of sample average $\hat{X}_M$ stated in \eqref{eq::firstdefestimator} is lower bounded by $\sigma^2\Vert X\Vert_2^2$ for $M\to \infty$ and for all $X\in\mathcal{D}\setminus \{\textbf{0}_N\}$.
\end{aproposition}
\begin{proof}
    Let the iteration step be fixed and observe the sample mean $\hat{X}_M$ given in \eqref{eq::firstdefestimator} with distribution described in \eqref{eq::distributionestimatornew}.
    Without loss of generality (W.\@l.\@o.\@g.\@), we assume $X\neq \textbf{0}_N$ and that $\frac{\partial g(X,f(X))}{\partial f(X)_j}>0$ for all $j\in\{1,...,M\}$.
    Then, consider the $i$th entry of $\hat{X}_M$, $i \in \{1,...,N\}$, particularly, its variance given by 
    \begin{equation}
        \var(\hat{X}_{M,i}) = \frac{1}{M^2} \sum_{j=1}^M \frac{C^2\sigma^2}{\beta_C(X)^2\frac{\partial g(X,f(X))}{\partial f(X)_j}^2}, \label{eq::varianceentrywisewlln}
    \end{equation}
    using the definition of the scale $s_j$ \eqref{eq::definitionscalingfactor}.
   It is easy to see that \eqref{eq::varianceentrywisewlln} decreases for decreasing $\sum_{j=1}^M \frac{1}{\beta_C(X)^2\frac{\partial g(X,f(X))}{\partial f(X)_j}^2}$.
   Therefore, minimising
    \begin{align}
     \sum_{j=1}^M \frac{1}{\beta_C(X)^2\frac{\partial g(X,f(X))}{\partial f(X)_j}^2}
     \label{eq::minimalsumwlln}
    \end{align}
    with respect to $\frac{\partial g(X,f(X))}{\partial f(X)_1},...,\frac{\partial g(X,f(X))}{\partial f(X)_M}$ minimises the variance in \eqref{eq::varianceentrywisewlln}.
    Note that doing so does not affect the multiplicative term $C^2\sigma^2/M$.
    By definition of the global concatenated gradient $G_X$ (see \eqref{eq::globalgradient}), the squared norm of $G_X$ is given by
        \begin{align}
        &\Vert G_X \Vert_2^2  \notag \\
        &= \sum_{j=1}^M \frac{\partial g(X,f(X))}{\partial f(X)_j}^2   \Vert X \Vert_2^2 + \sum_{j=1}^M \frac{\partial g(X,f(X))}{\partial b_j}^2 + \Vert G_{X,P}\Vert_2^2. \label{eq::normglobalgradientwlln}
    \end{align}
    Set $\Vert\text{Rest}\Vert_2^2:= \sum_{j=1}^M \frac{\partial g(X,f(X))}{\partial b_j}^2 +\Vert G_{X,P}\Vert_2^2$.
    Then, we can reformulate \eqref{eq::normglobalgradientwlln} and obtain the following constraint regarding $\frac{\partial g(X,f(X))}{\partial f(X)_1},...,\frac{\partial g(X,f(X))}{\partial f(X)_M}$:
    \begin{align}
        \sum_{j=1}^M \frac{\partial g(X,f(X))}{\partial f(X)_j}^2 = \frac{\Vert G_X \Vert_2^2 -\Vert\text{Rest}\Vert_2^2}{  \Vert X \Vert_2^2 }.  \label{eq::reformulationnormglobalgradientwlln}
    \end{align}
    Minimising the variance in \eqref{eq::varianceentrywisewlln} with respect to $\frac{\partial g(X,f(X))}{\partial f(X)_1},...,\frac{\partial g(X,f(X))}{\partial f(X)_M}$ under the constraint given in \eqref{eq::reformulationnormglobalgradientwlln}, does not affect the norm of the global gradient $G_X$ and, thus, it does not affect the value of $\beta_C(X)$. 
    Therefore, minimising \eqref{eq::minimalsumwlln} with respect to $\frac{\partial g(X,f(X))}{\partial f(X)_1},...,\frac{\partial g(X,f(X))}{\partial f(X)_M}$ under the constraint given in \eqref{eq::reformulationnormglobalgradientwlln} is equivalent to minimising 
    $\sum_{j=1}^M \frac{1}{\frac{\partial g(X,f(X))}{\partial f(X)_j}^2}$ with respect to $\frac{\partial g(X,f(X))}{\partial f(X)_1},...,\frac{\partial g(X,f(X))}{\partial f(X)_M}$ under \eqref{eq::reformulationnormglobalgradientwlln}.
    Hence, setting $y_j =\frac{\partial g(X,f(X))}{\partial f(X)_j}^2 $, for $j \in \{1,...,M\}$, we have an optimisation problem of the following form:
        \begin{align}
        \text{Minimise} \;
        \sum_{j=1}^M &\frac{1}{y_j}\; \text{ for } \\
        \sum_{j=1}^M y_j = \frac{\Vert G_X \Vert_2^2 -\Vert\text{Rest}\Vert_2^2}{  \Vert X \Vert_2^2 }\; &\text{ and } \;y_1,...,y_M>0,
        \label{eq::minimisationproblemwlln}
    \end{align}
    \eqref{eq::minimisationproblemwlln} is a well-known minimisation problem with solution given by $y_j = \frac{\Vert G_X \Vert_2^2-\Vert\text{Rest}\Vert_2^2}{M \Vert X \Vert_2^2 }$ for all $j \in \{1,...,M\}$.
    However, if needed, a proof of the statement can be obtained using the gradient of the function in \eqref{eq::minimisationproblemwlln} to construct the direction of the steepest descent and combining this with the given constraints in \eqref{eq::minimisationproblemwlln}.
    Hence, setting 
    \begin{equation}
        \frac{\partial g(X,f(X))}{\partial f(X)_j}^2= \frac{\Vert G_X \Vert_2^2-\Vert\text{Rest}\Vert_2^2}{M \Vert X \Vert_2^2 } \label{eq::solutionminimisationproblemwlln}
    \end{equation}
    for all $j \in \{1,...,M\}$, minimises the variance given in \eqref{eq::varianceentrywisewlln} with respect to $\frac{\partial g(X,f(X))}{\partial f(X)_1},...,\frac{\partial g(X,f(X))}{\partial f(X)_M}$. 

    We insert the choice \eqref{eq::solutionminimisationproblemwlln} into the variance \eqref{eq::varianceentrywisewlln} and obtain
    \begin{align}
        \var(\hat{X}_{M,i}) &\geq \frac{1}{M^2} \sum_{j=1}^M \frac{C^2\sigma^2M \Vert X \Vert_2^2 }{\beta_C(X)^2(\Vert G_X \Vert_2^2-\Vert\text{Rest}\Vert_2^2))}\notag \\
        &= \frac{C^2\sigma^2 \Vert X \Vert_2^2 }{\beta_C(X)^2(\Vert G_X \Vert_2^2-\Vert\text{Rest}\Vert_2^2))}\label{eq::varianceentrywiseminimalwlln},
    \end{align}
    for $i \in \{1,...,N\}$.

    Recall the definition of the clipping term $\beta_C(X)$ given in \eqref{eq::definitionbeta}.
    Using \eqref{eq::solutionminimisationproblemwlln}, we can see that the norm of the global gradient $\Vert G_X\Vert_2$ is linearly increasing in $M$.
    Thus, there exist $\hat{M}$, such that for all $M\geq \hat{M}$, $\Vert G_X\Vert_2 \geq C$ and $\beta_C(X) = \frac{C}{\Vert G_X\Vert_2}$. 
    Hence, by \eqref{eq::varianceentrywiseminimalwlln}
    \begin{align}
        &\lim_{M \to \infty} \var(\hat{X}_{M,i}) \\
        &\geq C^2 \sigma^2 \Vert X\Vert_2^2 \cdot \lim_{M \to \infty} \frac{1}{\beta_C(X)^2(\Vert G_X \Vert_2^2-\Vert\text{Rest}\Vert_2^2))} \notag\\
        &=C^2 \sigma^2 \Vert X\Vert_2^2 \cdot \lim_{M \to \infty} \frac{1}{C^2-\frac{\Vert\text{Rest}\Vert_2^2}{\Vert G_X \Vert_2^2}}\notag\\
        &= \sigma^2 \Vert X\Vert_2^2. \label{eq::convergencevarianceentrywiseminimalwlln}
    \end{align}
    Equality \eqref{eq::convergencevarianceentrywiseminimalwlln} holds because $\Vert \text{Rest}\Vert_2$ is independent of $M$.
    Lastly, if $X\neq \textbf{0}_N$, then $\sigma^2 \Vert X\Vert_2^2 >0$.
\end{proof}

\propminimalvarianceestimatornew*
\begin{proof} \label{proofpropminimalvarianceestimatornew}
    Let the iteration step be fixed and observe the sample mean $\hat{X}_M$ given in \eqref{eq::firstdefestimator} with distribution described in \eqref{eq::distributionestimatornew}.
    Then, consider the $i$th entry of $\hat{X}_M$, particularly, its variance given by 
    \begin{equation}
        \var(\hat{X}_{M,i}) = \frac{1}{M^2} \sum_{j=1}^M \frac{C^2\sigma^2}{\beta_C(X)^2\frac{\partial g(X,f(X))}{\partial f(X)_j}^2}, \label{eq::varianceentrywiseminimalstepone}
    \end{equation}
    for $i \in \{1,...,N\}$ using the definition of $s_j$ \eqref{eq::definitionscalingfactor}.
    W.\@l.\@o.\@g.\@, let $X\neq \textbf{0}_N$.
    We have shown in the proof of Proposition \ref{prop::norealizableunbiasedestimator} that the choice
    \begin{equation}
   \frac{\partial g(X,f(X))}{\partial f(X)_j}^2= \frac{\Vert G_X \Vert_2^2-\Vert\text{Rest}\Vert_2^2}{M \Vert X \Vert_2^2 } \label{eq::important}
    \end{equation}
    for all $j \in \{1,...,M\}$, minimises the variance given in \eqref{eq::varianceentrywiseminimalstepone} with respect to $\frac{\partial g(X,f(X))}{\partial f(X)_1},...,\frac{\partial g(X,f(X))}{\partial f(X)_M}$. 
    Let us set 
    \begin{equation}
       \frac{\partial g(X,f(X))}{\partial f(X)_1}^2= \frac{\Vert G_X \Vert_2^2-\Vert\text{Rest}\Vert_2^2}{M \Vert X \Vert_2^2 },
    \end{equation} 
and insert this choice into the variance \eqref{eq::varianceentrywiseminimalstepone}:
    \begin{equation}
        \var(\hat{X}_{M,i}) \geq \frac{C^2\sigma^2}{M}  \frac{1}{\beta_C(X)^2\frac{\partial g(X,f(X))}{\partial f(X)_1}^2}, \label{eq::varianceentrywiseminimalstepthree}
    \end{equation}
    for $i \in \{1,...,N\}$.
    If $\Vert G_X \Vert_2$ is fixed, it follows from \eqref{eq::important} that $\frac{\partial g(X,f(X))}{\partial f(X)_1}^2$ increases with decreasing norm $\Vert\text{Rest}\Vert_2$.
    However, $\Vert\text{Rest}\Vert_2$ cannot be bounded or quantified  for any iteration step without specific knowledge of the neural network.
    Thus, the adversary cannot minimise $\Vert\text{Rest}\Vert_2$ without manipulating some layers of the network. 
    If they manipulate these layers, we see that \eqref{eq::varianceentrywiseminimalstepthree} is minimal whenever $\frac{\partial g(X,f(X))}{\partial f(X)_1}$ is maximal, i.\@e.\@, whenever $\Vert\text{Rest}\Vert_2 = 0$.
    $\Vert\text{Rest}\Vert_2 = 0$ occurs for all $X\in \mathcal{D}$ and all iteration steps when the adversary replaces the entire network by the linear layer $f$ (see \Cref{sec:inversionattack}, specifically Equation \eqref{eq::linearlayer}) and sets the bias term $b$ to be equal to $\textbf{0}_M$.
    In such a case the neural network is given by the linear layer $f(X) = WX$ and a loss function which we denote by $\mathcal{L}: \mathbb{R}^N\times\mathbb{R}^M \to \mathbb{R}$.
    As a consequence, $\frac{\partial g(X,f(X))}{\partial f(X)_1} = \frac{\partial \mathcal{L}(X,f(X))}{\partial f(X)_1}$.
    In particular, \eqref{eq::important} implies
    \begin{equation}
        \frac{\partial g(X,f(X))}{\partial f(X)_1}^2= \frac{\partial \mathcal{L}(X,f(X))}{\partial f(X)_1}^2= \frac{\Vert G_X \Vert_2^2}{M \Vert X \Vert_2^2 }.
        \label{eq::choice}
    \end{equation}    
    Inserting \eqref{eq::choice} into the right hand side of \eqref{eq::varianceentrywiseminimalstepthree} further bounds the variance $\var(\hat{X}_{M,i})$: 
    \begin{equation}
        \var(\hat{X}_{M,i}) 
        \geq  \frac{C^2 \sigma^2 \Vert X \Vert_2^2}{\beta_C(X)^2\Vert G_X \Vert_2^2} 
        \label{eq::varianceentrywiseminimalstepfour},
    \end{equation}    
    for all $i \in \{1,...,N\}$.

    Now, we observe the lower bound in \eqref{eq::varianceentrywiseminimalstepfour}.
    Naturally, the right hand side of \eqref{eq::varianceentrywiseminimalstepfour} is lowest when the denominator in \eqref{eq::varianceentrywiseminimalstepfour} is highest.
    By definition of the clipping term $\beta_C(X)$ (see \eqref{eq::definitionbeta}), the product $\beta_C(X)^2\Vert G_X \Vert_2^2$ is upper bounded by $C^2$, delivering
     \begin{equation}
        \var(\hat{X}_{M,i}) 
        \geq   \sigma^2 \Vert X \Vert_2^2
        \label{eq::varianceentrywiseminimalstepfive},
    \end{equation}    
    for $i \in \{1,...,N\}$.   
    In particular, no change in the parameters or architecture of the network can increase the product $\beta_C(X)^2\Vert G_X \Vert_2^2$ beyond $C^2$ to further decrease the lower bound given in \eqref{eq::varianceentrywiseminimalstepfive}.
    Therefore, we assume, the adversary chooses $M$ and $ \frac{\partial \mathcal{L}(X,f(X))}{\partial f(X)_1}$ such that that $\Vert G_X \Vert_2^2\geq C^2$ for as many data points $X$ as possible.
    Using \eqref{eq::choice}, $\Vert G_X \Vert_2^2 \geq C^2$ implies
    \begin{equation}
       M \frac{\partial \mathcal{L}(X,f(X))}{\partial f(X)_1}^2 \geq \frac{C^2}{\Vert X\Vert_2^2}. \label{eq::varianceentrywiseminimalstepseven}
    \end{equation}
    Next, we consider two cases, when $\min_{X \in  \mathcal{D}}\Vert X\Vert_2>0$ and $\min_{X \in  \mathcal{D}}\Vert X\Vert_2=0$.
    If $\min_{X \in  \mathcal{D}}\Vert X\Vert_2>0$, then 
    \begin{align}
        \Vert G_X \Vert_2^2 &\geq C^2 && \forall X \in \mathcal{D} \notag\\
        \iff M \frac{\partial \mathcal{L}(X,f(X))}{\partial f(X)_1}^2 &\geq \frac{C^2}{\min_{X \in  \mathcal{D}}\Vert X\Vert_2^2}  &&\forall X \in \mathcal{D}. \label{eq::varianceentrywiseminimalstepsix}
    \end{align}    
    $M$, $\min_{X \in  \mathcal{D}}\Vert X\Vert_2$ and $C$ are fixed during training and do not changed from iteration to iteration. 
    Thus, \eqref{eq::varianceentrywiseminimalstepsix} holds for all $X \in \mathcal{D}$ if $\frac{\partial \mathcal{L}(X,f(X))}{\partial f(X)_1}$ is constant for all $X$ and all iteration steps, implying $\mathcal{L}$ is an affine function of $f(X)$.
    If $C>\min_{X \in  \mathcal{D}}\Vert X\Vert_2$, then choosing $M\geq \left\lceil\frac{C}{\min_{X \in  \mathcal{D}}\Vert X\Vert_2}\right\rceil$ and the loss function $\mathcal{L}:\mathbb{R}^N\times\mathbb{R}^M\to\mathbb{R}$ to be $\mathcal{L}(X,f(X))= \textbf{1}_M^Tf(X)$, where $\textbf{1}_M$ is the $M$-dimensional 1-vector, delivers the sample average $\hat{X}_M$ with the lowest variance per entry given by 
    \begin{equation}
        \var(\hat{X}_{M,i}) 
       = \sigma^2 \Vert X \Vert_2^2, 
    \end{equation}   
    where $\lceil \cdot \rceil$ denotes the function that rounds up its argument to the nearest integer.
    If $C \leq \min_{X \in  \mathcal{D}}\Vert X\Vert_2$, then choosing $M\geq1$ and the loss function to be $\mathcal{L}(X,f(X))= f(X)$ delivers the sample average $\hat{X}_M$ with the lowest variance per entry given by 
    \begin{equation}
        \var(\hat{X}_{M,i}) 
=\sigma^2 \Vert X \Vert_2^2 \geq \sigma^2 C^2.
    \end{equation}      
All in all, we conclude that if $\min_{X \in  \mathcal{D}}\Vert X\Vert_2>0$, replacing the subpart of the neural network given by $g$ by the loss function $\mathcal{L}(X,f(X))= \textbf{1}_M^Tf(X)$ and setting $M \geq \max\left( 1, \left\lceil \frac{C}{\min_{X\in \mathcal{D}}\Vert X\Vert_2}\right\rceil\right)$ minimises the variance $\var(\hat{X}_{M,i})$ for all $i\in\{1,...,N\}$.
    
    If $\min_{X \in  \mathcal{D}}\Vert X\Vert_2=0$, then there is no choice for $\frac{\partial \mathcal{L}(f((X))}{\partial f(X)_1}$ or $M$ such that \eqref{eq::varianceentrywiseminimalstepseven} holds for all $X\in\mathcal{D}$.
    However, in such a case, w.\@l.\@o.\@g.\@, we assume that the adversary sets the loss function to be $\mathcal{L}(X,f(X))= f(X)$ and chooses $M$ to ensure that \eqref{eq::varianceentrywiseminimalstepseven} holds for all $X\in\mathcal{D}$ with $X\neq \textbf{0}_N$.
    In such a case, the adversary sets $M= \left\lceil \frac{C}{\min_{X\in \mathcal{D}\setminus\{\textbf{0}_N\}}\Vert X\Vert_2}\right\rceil$, analogously as the argumentation above.
\end{proof}

The following lemma serves as an auxiliary result to obtain \Cref{prop::efficientestimatornew}:
\begin{alemma}\label{prop::efficientestimatorentrywisenew}
    For all $j \in \{1,...,N\}$, the $j$th entry $\hat{x}_j$ of the estimator $\hat{X}$ is a \textit{(fully) efficient} estimator for the $j$th entry $x_j$ of the target $X$.
\end{alemma}

\begin{proof}
The estimator $\hat{X}$ given in \eqref{eq::lastestimatorplease} is normally distributed with mean $X$ and covariance matrix given by $\sigma^2\Vert X\Vert_2^2 I_N$.
Let $\hat{x}_i$, $i \in \{1,...,N\}$, denote the $i$th entry of $\hat{X}$.
Then, by distribution of $\hat{X}$, $\hat{x}_1,...,\hat{x}_N$ are independent, normally distributed with mean $x_1,...,x_N$, respectively, and same variance given by $\sigma^2\Vert X\Vert_2^2 $.
Thus, for all $i \in \{1,...,N\}$, $\hat{x}_i$ is an unbiased estimator of the $i$th entry of the target $X$.

Moreover, applying the Cram\'{e}r-Rao bound for scalar unbiased estimators, we compute a lower bound for the variance of any the estimator of $\hat{x}_i$, $i \in \{1,...,N\}$:
\begin{equation} \label{eq::fisherinformationmatrimultivariate}
    \var_X\left(\hat{x}_i\right) \geq I(x_i)^{-1} = \sigma^2\Vert X\Vert_2^2, 
\end{equation}
where $I(x_i)$ denotes the Fisher information matrix that measures the amount of information the rescaled, observable normally distributed random variables $\hat{x}_i$ carries about its unknown mean $x_i$.
Since this matrix is well-known in literature, we do not provide a proof for the right hand side of the equality in \ref{eq::fisherinformationmatrimultivariate}.

Since for all $i \in \{1,...,N\}$, $\hat{x}_i$ is an unbiased estimator of $x_i$ that achieves the Cram\'{e}r-Rao bound, it is a \textit{(fully) efficient} estimator of $x_i$ achieving the smallest variability in terms of the variance.
\end{proof}

\propefficientestimatornew*
\begin{proof}
    Let $\hat{Y}=(\hat{y}_1,...,\hat{y}_N)^T$ denote any estimator of $X$.
    Then,
    \begin{align}
        &\ev_X[\MSE_X(X,\hat{Y})] \\
        &= \ev_X \left[ \frac{|| X - \hat{Y}||_2^2}{N} \right] \notag \\ 
        &= \frac{1}{N} \sum_{i =1}^N \ev_X \left[  (x_i-\hat{y}_i)^2\right] \notag \\ 
        &= \frac{1}{N} \sum_{i =1}^N \ev_{x_i} \left[  (x_i-\hat{y}_i)^2\right] \notag \\ 
        &= \frac{1}{N} \sum_{i =1}^N \left( \ev_{x_i} \left[  x_i-\hat{y}_i\right]^2 + \var_{x_i}(x_i-\hat{y}_i) \right) \notag \\ 
        &= \frac{1}{N} \sum_{i =1}^N \left( \Bias_{x_i}(x_i, \hat{y}_i)^2 + \var_{x_i}(\hat{y}_i) \right). \label{eq::mseestimatorbiased}
    \end{align}
    For all unbiased estimators, the expected $\MSE$, as given in \ref{eq::mseestimatorbiased}, is solely determined by the sum of the variances of each entry $\hat{y}_i$.
    Therefore, by Lemma \ref{prop::efficientestimatorentrywisenew}, $\hat{X}$ is the unbiased estimator that minimises \ref{eq::mseestimatorbiased}.
    In other words, $\hat{X}$ is the unbiased estimator that achieves the lowest expected $\MSE$.
    Such estimators are called \textit{minimum variance unbiased estimators} in the literature and achieve the smallest variability in terms of the variance.
    
    Lastly, we compute the expected $\MSE$ between $X$ and $\hat{X}$:
    \begin{align*}
        \ev_X[\MSE_X(X,\hat{X})] = \frac{1}{N} \sum_{i =1}^N \var_{x_i}(\hat{x}_i) = \sigma^2 \Vert X \Vert_2^2.
    \end{align*}    

Since the adversary only observes one privatized version of the gradient $\tilde{\nabla}_W$ they can use to construct an estimator for $X$, it is easy to see that $\hat{X}$ is a sufficient statistic for estimating $X$.
Moreover, by Proposition \ref{prop::efficientestimatornew}, $\hat{X}$ is the unbiased estimator, which uses the sufficient statistic $\tilde{\nabla}_W$ as input, that achieves the lowest expected $\MSE_X(X, \hat{X})$.
Such unbiased estimators achieve the lowest possible $\MSE$ and have the smallest variability in terms of their variance.
Thus, using the $\MSE$ as an optimality criterion, $\hat{X}$ is the optimal estimator for $X$.
Since lower values of $\MSE_X(X, \hat{X})$ denote high similarity between $X$ and $\hat{X}$, we conclude that $\hat{X}$ is the best achievable reconstruction for $X$. 
\end{proof}

\distrmsefirstcase*
\begin{proof}
   We can compute the $\MSE$ between $X$ and its reconstruction $\hat{X}$ as the mean error over their components: 
\begin{align} 
    \MSE_X(X,\hat{X}) 
    &= \frac{1}{N}\sum_{i=1}^N(x_i-\hat{x}_i)^2 \notag \\
&= \frac{1}{N}\sum_{i=1}^N\tilde{\xi_i}^2, \qquad \tilde{\xi_i}\sim \mathcal{N}(0,\sigma^2\Vert X\Vert_2^2) \notag \\
    &= \frac{1}{N}\sum_{i=1}^N \sigma^2\Vert X\Vert_2^2\left(\frac{1}{\sigma\Vert X\Vert_2}\tilde{\xi}_i\right)^2 \notag \\
    &= \frac{\sigma^2\Vert X\Vert_2^2}{N}\sum_{i=1}^N \rho_i^2,\label{eq::mse_distribution_proof1}
\end{align}
where $\rho_i := \frac{1}{\sigma\Vert X\Vert_2}\xi_i$, for $i \in \{1,...,N\}$. 
If $X$ is fixed, $\rho_1, ... ,\rho_N$ are pairwise independent random variables with $\rho_i \sim \mathcal{N}\left(0, 1\right)$ for all $i \in \{1,...,N\}$.
Hence,
\begin{align}  \label{eq::mse_distribution_proof2}
    \sum_{i=1}^N \rho_i^2 &\sim \chi^2_N,
\end{align}
where $\chi^2_N$ denotes the central chi-squared distribution with $N$ degrees of freedom.
Thus,
\begin{equation*}
    \MSE_X(X, \hat{X}) \overset{d}{=} \frac{\sigma^2\Vert X\Vert_2^2}{N} \cdot Y \quad \text{with} \quad Y \sim \chi^2_N.
\end{equation*}
Then, for $\eta$ given
\begin{align*}
   \mathbb{P}_{\hat{X}}\left(\MSE_X(X, \hat{X}) \leq \eta\right) &= \gamma \\
    \iff \mathbb{P}_{Y}\left(\frac{\sigma^2\Vert X\Vert_2^2}{N} \cdot Y \leq \eta\right) &= \gamma.
\end{align*}
Since $Y$ is a centered chi-squared distributed random variable with $N$ degrees of freedom, its cumulative distribution function can be computed via the regularized gamma function $\Gamma_R$.
Hence,  
\begin{align*}
    &\mathbb{P}_{Y}\left(\frac{\sigma^2\Vert X\Vert_2^2}{N} \cdot Y \leq \eta\right) \\
    &= P_{Y}\left(Y \leq \frac{ N \eta}{\sigma^2\Vert X\Vert_2^2}\right) \\
    &=\Gamma_R\left(\frac{N}{2}, \frac{N \eta}{2 \sigma^2\Vert X\Vert_2^2}\right) 
\end{align*}
implies
\begin{align}
    \mathbb{P}_{\hat{X}}\left( \MSE_X(X,\hat{X}) \leq \eta \right) = \Gamma_R \left(\frac{N}{2}, \frac{N \eta}{2\sigma^2\Vert X\Vert_2^2}\right).
\end{align}
\end{proof}

\reromsefirstcase*

\begin{proof}

Consider the CDF of the $\MSE_X(X,\hat{X})$ for $X\in\mathcal{D}\setminus\{\textbf{0}_N\}$ given Equation \eqref{eq::almostrero} in \Cref{prop::distrmsefirstcase}.
Since the regularized gamma function is increasing in its second argument, using \Cref{prop::distrmsefirstcase}, it follows that 
\begin{equation} \label{eq::reromsefirstcaseforallx}
    \begin{split}
    &\mathbb{P}_{\hat{X}}\left( \MSE_X(X,\hat{X}) \leq \eta \right) \\
    &= \Gamma_R \left(\frac{N}{2}, \frac{N \eta}{2 \sigma^2\Vert X\Vert_2^2}\right) \\
    &\leq \Gamma_R \left(\frac{N}{2}, \frac{N \eta}{2\sigma^2 \min_{X\in\mathcal{D}\setminus\{\textbf{0}_N\}}\Vert X \Vert_2^2}\right),
    \end{split}
\end{equation} 
for all reconstruction target points $X\in\mathcal{D}\setminus\{\textbf{0}_N\}$.
Moreover, due to the optimality of the estimator $\hat{X}$, it holds
\begin{equation}  \label{eq::reromseoptimalestimatorupperbound}
    \begin{split}
    &\mathbb{P}_{Y_X'}\left( \MSE_X(X,Y_X') \leq \eta \right) \\
    &\leq \mathbb{P}_{\hat{X}}\left( \MSE_X(X,\hat{X}) \leq \eta \right)
   \end{split}
\end{equation}
for all possible reconstructions $Y_X'$.
Therefore, combining \eqref{eq::reromsefirstcaseforallx} and \eqref{eq::reromseoptimalestimatorupperbound}, we conclude that $\mathcal{M}$ \eqref{eq::modifieddpmechanism} is $(\eta, \gamma(\eta))$-reconstruction robust with respect to the $\MSE$ for any reconstruction and $\gamma(\eta) = \Gamma_R\left(\frac{N}{2}, \frac{N \eta}{2\sigma^2 \min_{X\in\mathcal{D}\setminus\{\textbf{0}_N\}}\Vert X\Vert_2^2} \right)$.

\end{proof}

\psnrdistributionfirstcase*
\begin{proof}
The cumulative distribution function (CDF) of the $\PSNR$ can be calculated using the CDF of the $\MSE$.
In particular, this implies that we can also compute probabilistic bounds for the $\PSNR$ using \Cref{prop::distrmsefirstcase}.
Let $\eta$ be given.
Then,
\begin{alignat*}{3}
        & \PSNR_X(X,\hat{X}) \geq \eta \\
        \iff& 10 \log_{10}\left( \frac{(\max(X)-\min(X))^2}{\MSE(X,\hat{X})}\right) \geq \eta \\
        \iff& \log_{10}((\max(X)-\min(X))^2) -\log_{10}(\MSE(X,\hat{X}) \geq \frac{\eta}{10} \\
        \iff& \log_{10}((\max(X)-\min(X))^2)-\frac{\eta}{10} \geq \log_{10}(\MSE(X,\hat{X}))  \\
        \iff&  (\max(X)-\min(X))^210^{-\frac{\eta}{10}} \geq \MSE(X,\hat{X}).   
\end{alignat*}
Thus, setting
\begin{equation}
    \hat{\eta}(\eta) =10^{-\frac{\eta}{10}}(\max(X)-\min(X))^2,
\end{equation}
it follows from \Cref{prop::distrmsefirstcase} that
\begin{align}
        &\mathbb{P}_{\hat{X}}(\PSNR_X(X,\hat{X}) \geq \eta) \notag  \\
        &=\mathbb{P}_{\hat{X}}(\MSE_X(X,\hat{X}) \leq \hat{\eta}(\eta)) \notag\\ 
        &\leq \Gamma_R \left(\frac{N}{2}, \frac{N \hat{\eta}(\eta)}{2}\frac{1}{\sigma^2 \min_{X \in \mathcal{\mathcal{D}\setminus\{\textbf{0}_N\}}}\Vert X \Vert_2^2}\right). \label{eq::1000}
\end{align}

We note that \eqref{eq::1000} is still dependent on the target value $X$ due to $\hat{\eta}(\eta)$.
To remove this dependency, we find an upper bound for $\hat{\eta}(\eta)$:
\begin{equation*}
\begin{split}
    &\hat{\eta}(\eta) \\
    &=10^{-\frac{\eta}{10}}(\max(X)-\min(X))^2 \\
    &\leq 10^{-\frac{\eta}{10}}\left(\max_{X \in \mathcal{D}}\max(X)-\min_{X \in \mathcal{D}}\min(X)\right)^2\\
    &:= \tilde{\eta}(\eta).
\end{split}
\end{equation*}
Since the regularized gamma function $\Gamma_R$ is increasing with respect to the second argument, it follows:
\begin{equation*}
\begin{split}
    &\mathbb{P}_{\hat{X}}(\PSNR_X(X,\hat{X}) \geq \eta) \\
    &\leq \Gamma_R \left(\frac{N}{2}, \frac{N \tilde{\eta}(\eta)}{2\sigma^2\min_{X \in \mathcal{D}\setminus\{\textbf{0}_N\}}\Vert X \Vert_2^2}\right),
\end{split}
\end{equation*}
for all $X\in \mathcal{D}\setminus\{\textbf{0}_N\}$.
\end{proof}

\psnrrerofirstcase*
\begin{proof}
    On the one hand, by the optimality of the reconstruction $\hat{X}$ with respect to the $\PSNR$, it holds that 
    \begin{equation*}
        \mathbb{P}_{Y_X}(\PSNR_X(X,Y_X') \geq \eta) \leq \mathbb{P}_{\hat{X}}(\PSNR_X(X,\hat{X}) \geq \eta),
    \end{equation*}
    for any analytic reconstruction $Y_X'$.
    On the other hand, by \Cref{propo:psnrdistibutionfirstcase},
    \begin{equation*}
        \begin{split}
            &\mathbb{P}_{\hat{X}}(\PSNR_X(X,\hat{X}) \geq \eta) \\
            &= \mathbb{P}_{\hat{X}}(-\PSNR_X(X,\hat{X}) \leq -\eta) \\
            &\leq \Gamma_R \left(\frac{N}{2}, \frac{N \tilde{\eta}(\eta)}{2\sigma^2 \min_{X \in \mathcal{D}\setminus\{\textbf{0}_N\}}\Vert X \Vert_2^2}\right),
        \end{split}
    \end{equation*}
    for $\tilde{\eta}(\eta)$ as defined in \Cref{propo:psnrdistibutionfirstcase} and $\Gamma_R$ being the regularized gamma function.
    Therefore, using \Cref{def::rero}, we conclude that the DP mechanism $\mathcal{M}$ is $(-\eta, \tilde{\gamma}(\tilde{\eta}(\eta)))$-reconstruction robust for $\tilde{\gamma}(\tilde{\eta}(\eta))= \Gamma_R \left(\frac{N}{2}, \frac{N \tilde{\eta}(\eta)}{2\sigma^2\min_{X \in \mathcal{D}\setminus\{\textbf{0}_N\}}\Vert X \Vert_2^2}\right)$ with respect to the negative $\PSNR$, i.\@e.\@, $-\PSNR$.
    
    Lastly, since the regularized gamma function $\Gamma_R$ is increasing with respect to the second argument, if $C\leq \min_{X \in \mathcal{D}}\Vert X \Vert_2$, then the DP mechanism $\mathcal{M}$ \eqref{eq::modifieddpmechanism} is $(-\eta, \tilde{\gamma}'(\tilde{\eta}(\eta)))$-reconstruction robust for $\tilde{\gamma}'(\tilde{\eta}(\eta))= \Gamma_R \left(\frac{N}{2}, \frac{N \tilde{\eta}(\eta)}{2\sigma^2C^2}\right)$ with respect to the negative $\PSNR$, i.\@e.\@, $-\PSNR$.
\end{proof}

\convergenceundermultipletimesteps*
\begin{proof}
Let $\hat{X}_j$, for $j \in \{1,...,k\}$, denote reconstructions of the same data sample $X$ obtained separately and independently by performing multiple attacks. 
The distribution of each reconstruction $\hat{X}_j$ is given in Equation \ref{eq::lastestimatorplease} (\Cref{sec:optimalityattack}).
Let $\hat{x}_{j,i}$, for $j \in \{1,...,k\}$ denote the $i$th coordinate of $\hat{X}_j$.
Then, the expectation and variance of the $i$th component of $\hat{X}_\emph{avg}$, i.\@e.\@, $\hat{x_i}_{\emph{avg}}$, can be computed in the following way:
\begin{equation*}
    \ev[\hat{x_i}_{\emph{avg}}] = \ev\left[ \frac{1}{k}\sum_{j=1}^k\hat{x}_{j,i} \right] = \frac{1}{k} \sum_{j=1}^k\ev[\hat{x}_{j,i}] = x_i,
\end{equation*}
and
\begin{align*}
    &\var(\hat{x_i}_{\emph{avg}}) \\
    &= \var\left( \frac{1}{k}\sum_{j=1}^k\hat{x}_{j,i} \right) \\
    &= \frac{1}{k^2} \sum_{j=1}^k\var(\hat{x}_{j,i}) \\
    &= \frac{1}{k^2}k \sigma^2\Vert X\Vert_2^2.
\end{align*}

\end{proof}

\clearpage
\subsection{Experimental details}\label{sec::detailtightness}

The experiments in \Cref{sec::tightness} are designed to empirically validate the tightness of our theoretical bounds against practical reconstruction attacks across a range of conditions. All code is provided with the submission to ensure full reproducibility.

\subsubsection{Setup}
Our experimental setup mirrors the theoretical model from \Cref{sec::results}. The total gradient $G_{X,P}$ is composed of gradients from both a main reconstruction network $f$ and an auxiliary network $g$, as defined in \Cref{eq::globalgradient}. This ensures that the gradient component $G_P$, which does not depend on the private data $X$, is non-empty, contributing to the overall gradient norm $\Vert G_X\Vert_2$ and reflecting a more realistic, complex training scenario.

\subsubsection{Dataset and Preprocessing}
We use 500 randomly sampled images from the CIFAR-10 dataset \citep{krizhevsky2009learning}. To effectively visualize the probability distributions of the reconstruction error, we reduce the dimensionality of the input data. Specifically, images are converted to grayscale and resized to $2\times2$ pixels, resulting in an input dimensionality of $N=4$. This is necessary because in very high dimensions, the probability density functions of the error become extremely concentrated (or degenerate), making them difficult to plot and interpret. To ensure the uniform applicability of our bounds across all samples, we normalize and clip each image to have an $\ell_2$-norm of $1.01$, with a gradient clipping norm of $C=1$.

\subsubsection{Models and Evaluation}
To demonstrate that our bounds hold irrespective of architectural complexity, we evaluate three distinct scenarios:
\begin{itemize}
    \item \textit{Optimal:} A minimal linear network that directly implements the optimal attack from \Cref{sec::results}.
    \item \textit{Linear (1M params):} A simple linear network with a large number of additional, non-contributing parameters to simulate a more complex model.
    \item \textit{ResNet-101:} A standard, deep convolutional network (ResNet-101) to represent a realistic, state-of-the-art architecture.
\end{itemize}

For each setup, we perform reconstructions across a logarithmic range of noise multipliers ($\sigma$) and measure the resulting Mean Squared Error (MSE) and Peak Signal-to-Noise Ratio (PSNR). As shown in \Cref{fig:tightness}, the empirical distributions of these error metrics closely match our theoretical bounds, confirming their tightness. Furthermore, we demonstrate in \Cref{fig:tightness_high_n} that the bounds become even tighter for higher-dimensional data, confirming that our low-dimensional visualization represents a conservative view of the bounds' accuracy.
\begin{figure}[H]
    \centering
    \includegraphics[width=\linewidth]{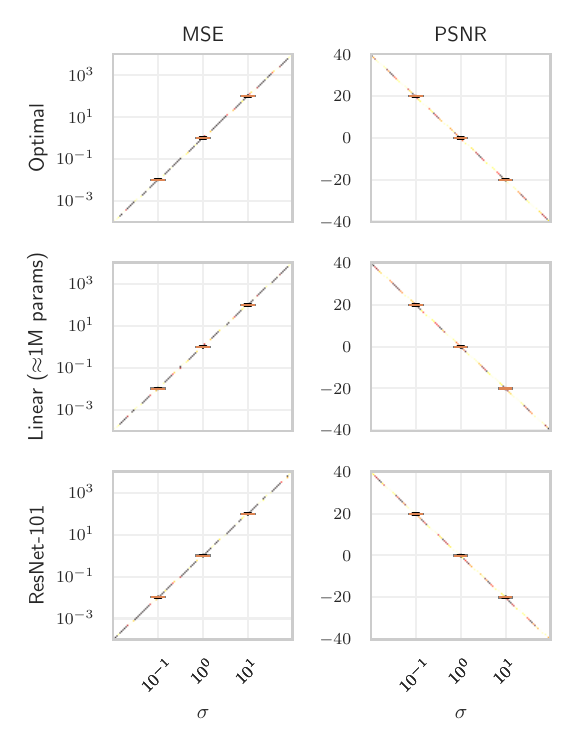}
    \caption{
    Comparing reconstruction error distributions to empirical results for high-dimensional data. 
    Setup equivalent to \Cref{fig:tightness}, but with $N=3\,072$. 
    }
    \label{fig:tightness_high_n}
\end{figure}

\end{document}